\documentclass[a4paper,fleqn]{cas-dc}
\usepackage[numbers]{natbib}
\usepackage[switch]{lineno}
\usepackage{booktabs}
\usepackage{multirow}
\usepackage{graphicx}
\usepackage{amsmath,amsthm}
\usepackage{float}
\usepackage{array}
\usepackage{listings}
\usepackage{fvextra}
\setcitestyle{maxcitenames=4}

\newtheorem{definition}{Definition}
\newtheorem{lemma}{Lemma}
\newtheorem{theorem}{Theorem}
\newtheorem{corollary}{Corollary}

\usepackage{xcolor}
\usepackage{ifthen}

\newboolean{showcomments}    
\newboolean{shownewcontent}  
\newboolean{showrevision}
\newboolean{draft4review}   

\setboolean{showcomments}{true}
\setboolean{shownewcontent}{false}
\setboolean{showrevision}{false}
\setbool{draft4review}{false}
\ifthenelse{\boolean{draft4review}}{
  \settopmatter{printacmref=false}
  \renewcommand\footnotetextcopyrightpermission[1]{}
}{
}

\ifthenelse{\boolean{showcomments}}{%
  \newcommand{\NB}[3]{%
    {\colorbox{#3}{\bfseries\sffamily\scriptsize\textcolor{white}{#1}}}%
    {\textcolor{#3}{\sffamily\small$\langle$\textit{#2}$\rangle$}}%
  }%
}{%
  \newcommand{\NB}[3]{}%
}

\ifthenelse{\boolean{shownewcontent}}{%
}{%
}

\ifthenelse{\boolean{showrevision}}{%
  \newcommand{\revise}[1]{{\color{blue}#1}}%
}{%
  \newcommand{\revise}[1]{#1}%
}

\newcounter{rq}
\newcounter{rqfinding}

\newcommand{\setrq}[1]{%
  \setcounter{rq}{#1}%
  \setcounter{rqfinding}{0}%
}

\newcommand{\Finding}[1]{%
  \refstepcounter{rqfinding}%
  \par\vspace{0.35\baselineskip}%
  \noindent\textbf{Finding \arabic{rq}.\arabic{rqfinding}:}~#1\par
}

\newcommand{\Implication}[1]{%
  \noindent\textbf{Implication:}~#1\par
}

\def\tsc#1{\csdef{#1}{\textsc{\lowercase{#1}}\xspace}}
\tsc{WGM}
\tsc{QE}
\tsc{EP}
\tsc{PMS}
\tsc{BEC}
\tsc{DE}

\begin{document}
\let\WriteBookmarks\relax
\def\floatpagepagefraction{1}
\def\textpagefraction{.001}
\shorttitle{Logic-Grounded Metamorphic Testing for Evaluating the Reasoning Reliability of LLMs}
\shortauthors{Zenghui Zhou et~al.}

\title [mode = title]{LGMT: Logic-Grounded Metamorphic Testing for Evaluating the Reasoning Reliability of LLMs}                      




\affiliation[1]{
organization={School of Automation Science and Electrical Engineering, Beihang University},
addressline={Xueyuan Road, Haidian District}, 
city={Beijing},
postcode={100191}, 
country={China}
}

\author[1]{Zenghui Zhou}
\credit{Conceptualization, Investigation, Code, Visualization, Data curation, Writing}

\author[1]{Man Li}
\credit{Investigation, Code, Visualization, Data curation}

\author[1]{Xiaoke Fang}
\credit{Investigation, Code, Data curation}

\author[1]{Xinyi Zhou}
\credit{Investigation, Code, Data curation}

\author[1]{Weibin Lin}
\credit{Code}

\author[1]{Zheng Zheng}\cormark[1]
\ead{zhengz@buaa.edu.cn}
\credit{Conceptualization, review \& editing, Funding acquisition}

\cortext[cor1]{Corresponding author}

\begin{abstract}
Large Language Models (LLMs) achieve strong performance on logical reasoning benchmarks, yet their reliability remains uncertain. Existing evaluations rely on static benchmarks, which fail to assess robustness under logically equivalent transformations and often overestimate reasoning capability.
We propose LGMT (Logic-Grounded Metamorphic Testing), an oracle-free framework that leverages first-order logic (FOL) to evaluate LLM reasoning. By deriving metamorphic relations from formal logical equivalences, LGMT constructs semantically invariant test cases and detects reasoning defects through cross-case consistency checking.
Experiments on six state-of-the-art LLMs show that LGMT exposes substantial hidden defects missed by traditional reference-based evaluations. We further find that models are particularly sensitive to symbol-level and conclusion-level variations, and that advanced prompting such as Few-shot CoT only partially mitigates these issues.
These results suggest that LLM evaluation should move beyond isolated correctness toward robustness under logical invariance. LGMT provides a principled and scalable approach for diagnosing reasoning failures.
\end{abstract}


\maketitle

\section{Introduction}
\label{sec:intro}

Large Language Models (LLMs) have been increasingly integrated into modern software systems, powering a wide range of applications such as automated question answering~\cite{deepseek2025deepseekv3,lewis2020retrievalaugmented,yue2025survey,zhuang2023toolqa}, code generation~\cite{deepseek-ai2024deepseekcoderv2,jimenez2024swebench,roziere2024code}, and intelligent assistants~\cite{guan2023intelligent,park2024thinking}. In many of these applications, LLMs are expected not only to generate fluent responses, but also to perform reliable logical reasoning over given facts, rules, and constraints~\cite{holliday2024conditional,luo2023logiglue,singh2024are}. This property is critical for LLM-integrated software systems, because inconsistent reasoning under semantically equivalent inputs may lead to unreliable decisions even when the model appears accurate on individual benchmark instances~\cite{huang2025survey, woesle2025systematic, xu2025are, wan2024logicasker}.
As LLMs increasingly act as knowledge-intensive components in software systems, rigorously testing whether their reasoning results remain stable under equivalent semantics is a prerequisite for ensuring reliability before practical deployment~\cite{ghosh2025logical,liu2025evaluating, yu2020reclor,zhang2025system}.

However, testing the logical reasoning reliability of LLMs faces fundamental challenges. Currently, the dominant testing paradigm relies heavily on reference-based evaluation using static benchmarks (e.g., FOLIO~\cite{han2024folio}, LogicBench~\cite{parmar2024logicbench}, and LogiQA~\cite{liu2021logiqa}). This paradigm inherently suffers from a severe \textit{Test Oracle Problem}: manually labeling expected logical outputs for diverse inputs is prohibitively expensive and fundamentally cannot scale to support extensive, automated testing. 
Moreover, the static nature of these datasets makes them highly susceptible to \textit{data contamination} (i.e., benchmark leakage)~\cite{xu2024benchmarking}, where LLMs may inadvertently memorize the test cases during pre-training. This may create a false illusion of reasoning reliability rather than reflecting their true out-of-distribution generalization capabilities. Consequently, static benchmarks often suffer from poor timeliness and truncated lifecycles. To accurately diagnose the true capabilities of LLMs, evaluators are forced into an exhausting cycle of continuously constructing entirely new, unseen datasets~\cite{qi2024large}. This perpetual "cat-and-mouse" game between benchmarks and model training incurs an unsustainable maintenance cost, further exacerbating the oracle problem~\cite{dziri2023faith}.
Furthermore, these static benchmarks typically formulate logical reasoning as binary, ternary, or multiple-choice classification tasks. Relying solely on such static labels for validation is intrinsically unreliable. Due to the restricted output space, LLMs can achieve seemingly high accuracy through random guessing or by exploiting spurious statistical artifacts (i.e., textual shortcuts) rather than engaging in genuine step-by-step logical deduction~\cite{saparov2022language}. This \textit{coincidental correctness} severely masks their underlying logical deficits~\cite{li2024drowzee}.

To address these limitations, recent studies have introduced Metamorphic Testing (MT) as a zero-resource, black-box testing methodology for LLMs~\cite{chen2019metamorphic, li2024drowzee, wu2025detecting, yang2025hallucinationa}. By examining whether the outputs of source and follow-up test cases satisfy predefined Metamorphic Relations (MRs), MT effectively alleviates the oracle problem~\cite{chen1998metamorphic}. Crucially, MT shifts the evaluation focus from individual test case to the preservation of invariant properties. This characteristic makes MT inherently more resilient to data contamination, as simply memorizing training samples is insufficient to guarantee consistency across diverse MRs. Furthermore, the cross-case validation mechanism in MT helps mitigate coincidental correctness, as it is significantly more challenging for LLMs to maintain consistency through random guessing than to pass a single static test.

Despite its promise, current MT approaches for NLP tasks predominantly rely on \textit{lexical-level} natural language perturbations, such as synonym substitution or syntactic paraphrasing~\cite{chen2021testing,cho2025metamorphic,wu2025detecting,yang2025hallucinationa}. Since natural language is inherently ambiguous, these informal input transformations frequently introduce \textit{semantic drift} or grammatical inconsistencies. Recent large-scale empirical studies demonstrate that such natural-language-based MRs often yield high false positive rates~\cite{cho2025metamorphic}. Consequently, when a test fails, developers struggle to distinguish whether the failure stems from a genuine logical defect in the LLM or a follow-up test case that inadvertently altered the original underlying logic.

To overcome these limitations, we propose \textbf{LGMT}, a logic-grounded metamorphic testing framework for systematically evaluating the logical reasoning reliability of LLMs. The key novelty of LGMT lies in grounding metamorphic testing in First-Order Logic (FOL): instead of relying on informal lexical mutations, we derive MRs directly from formally verified FOL equivalence laws, such as De Morgan's laws, contraposition, and quantifier transformations. This design elevates metamorphic testing from lexical-level perturbation to \emph{logic-preserving transformation}, providing a mechanism for generating follow-up test cases whose symbolic semantics are guaranteed by construction. As a result, LGMT substantially reduces the semantic drift that has long limited natural-language-based MT. More importantly, by validating whether LLMs preserve their judgments across logically equivalent test cases, LGMT moves evaluation beyond isolated correctness on static benchmarks and toward robustness under logical invariance. This allows the framework not only to expose hidden reasoning defects that traditional testing misses, but also to provide a formally grounded alternative for oracle-free reliability assessment of LLM reasoning.

To the best of our knowledge, this is among the first studies to ground MRs in FOL for systematically assessing the logical reasoning reliability of LLMs. The main contributions of this paper are summarized as follows:
\begin{itemize}
    \item We propose \textbf{LGMT}, a metamorphic testing framework that derives input transformations from FOL and evaluates whether LLM entailment judgments remain consistent across source and follow-up cases.

    \item We design 20 types of FOL-grounded metamorphic relations across formula, symbol, premise, and conclusion-level transformations, enabling systematic generation of logic-preserving follow-up test cases.

    \item We provide a formal analysis of MR-E1 with respect to PNNF-oriented normalization, showing that the core equivalence rules are sufficient for reaching a canonical normalized representation under the proposed rewrite procedure.

    \item We construct 76,298 candidate metamorphic groups from three reasoning benchmarks and evaluate 3,091 sampled groups across six LLMs and four prompting strategies.
    
    \item We empirically show that LLMs exhibit non-trivial inconsistency under FOL-grounded transformations, with particularly high sensitivity to symbol and conclusion-level variants, and an audit of sampled reports suggests a low false-report rate.
\end{itemize}

\section{Preliminaries}
\label{sec:preliminaries}
\revise{In this section, we first introduce first-order logic and logical entailment, then formalize MT for LLM reasoning evaluation.
We then present motivating examples to illustrate the limitations of current LLMs and to clarify why FOL-grounded MRs provide a more reliable alternative to lexical MRs.}

\subsection{First-Order Logic and Logical Entailment}

\textbf{Syntax of First-Order Logic.}
FOL is a formal symbolic system extensively utilized to represent complex facts and rules~\cite{cao2025advanced,pan2023logiclm}. Let $\mathcal{L}$ be a first-order language defined over predicates (representing properties and relations), constants (denoting specific entities), and variables. The formulas in $\mathcal{L}$ are recursively constructed from these components using a set of logical connectives, including negation ($\neg$), conjunction ($\wedge$), disjunction ($\vee$), implication ($\rightarrow$), and biconditional ($\leftrightarrow$)~\cite{han2024folio}. To enable quantification over entities within a domain, FOL further introduces universal ($\forall$) and existential ($\exists$) quantifiers. This compositional structure allows FOL to precisely capture the latent logic of natural language sentences, thereby translating linguistic forms into rigorous symbolic representations~\cite{jiang2025large, xu2025are}.

\textbf{Formalizing Logical Reasoning.}
In the context of reasoning evaluation, a test case $x$ is formally abstracted as a premise--conclusion pair $x = \langle \Gamma, q \rangle$. Here, $\Gamma = [p_1, p_2, \dots, p_m]$ constitutes a finite set of FOL formulas representing the given premises (i.e., background facts and deductive rules), and $q$ is an FOL formula denoting the target conclusion \cite{parmar2024logicbench}. The core objective of the LLM $f(\cdot)$ is to function as a deductive reasoner, determining whether the premises semantically entail the conclusion, denoted as $\Gamma \models q$. Formally, this entailment holds if and only if every interpretation (or world model) that satisfies all formulas in $\Gamma$ also satisfies $q$ \cite{olausson2023linc}. In this paper, we adopt a three-valued entailment setting:
\textit{True} iff $\Gamma \models q$,
\textit{False} iff $\Gamma \models \neg q$,
\textit{Unknown} otherwise.

\textbf{Logical Equivalence and Logical Semantic Preservation.}
To construct MRs, we leverage the foundational property of \textit{logical equivalence}. Two formulas $\phi$ and $\psi$ are logically equivalent, denoted as $\phi \equiv \psi$, if they have identical truth values under all possible interpretations~\cite{parmar2024logicbench}. In FOL, equivalence laws such as De Morgan's laws, the Law of Contraposition, and Double Negation Elimination constitute strict equivalence transformations.
Suppose we apply these laws to transform a source test case $\langle \Gamma, q \rangle$ into a follow-up test case $\langle \Gamma', q \rangle$. Because FOL mathematically guarantees $\Gamma \equiv \Gamma'$, it strictly preserves the invariance of the entailment relationship:
$$ \Gamma \models q \iff \Gamma' \models q $$

Consequently, FOL-grounded transformations preserve the underlying logical semantics. Since the transformed formulas are later translated into natural language, we further audit the generated follow-up cases empirically to account for possible realization errors. Under this design, reported metamorphic violations can be interpreted as logical reasoning defects unless they are attributable to translation drift or output parsing errors.

\subsection{Metamorphic Testing for Large Language Models}
\label{subsec:formalizing_mt}

In the paradigm of software testing, the System Under Test (SUT) is typically abstracted as a target function $f: X \rightarrow Y$, mapping an input space $X$ to an output space $Y$. Traditional testing methodologies rely heavily on a \textit{test oracle} to verify whether the actual execution output $f(x)$ matches the expected ground truth for a given input $x$~\cite{chen2019metamorphic}. However, for highly complex systems such as LLMs, acquiring an exact oracle for arbitrary inputs is prohibitively expensive, leading to the severe oracle problem~\cite{li2024drowzee,wu2025detecting}. MT alleviates this fundamental bottleneck by shifting the verification from individual test case to the properties among multiple executions.

The core of MT lies in defining MRs, which specify the necessary properties that the target function $f$ must satisfy across multiple executions. Formally, given a tuple of $n$ test inputs $(x_1, \dots, x_n)$ where $n \ge 2$, an MR is defined as a logical implication:
$$ \mathcal{R}_i(x_1, \dots, x_n) \Rightarrow \mathcal{R}_o(f(x_1), \dots, f(x_n)) $$
where $\mathcal{R}_i$ represents the \textit{input relation} among the inputs, and $\mathcal{R}_o$ represents the \textit{output relation} expected to hold among the corresponding outputs~\cite{cho2025metamorphic}. For the sake of conciseness in the following discussion, we denote a metamorphic relation as the pair $\langle \mathcal{R}_i, \mathcal{R}_o \rangle$.

In a standard MT workflow, testing begins with the selection of a \textit{source test case} $x_s$. A specific input transformation is then applied to derive a \textit{follow-up test case} $x_f$, strictly ensuring that the predefined input relation holds true (i.e., $\mathcal{R}_i(x_s, x_f) = \text{True}$). 
The combination of the source and follow-up test cases constitutes a \textit{metamorphic test group (MG)}\footnote{In classical MT literature, an MG encompasses the complete set of source and follow-up test cases required for one instance of checking the relevant MR. Since our MRs strictly involve a one-to-one mapping, we define the MG as a bi-tuple without loss of generality.}, denoted as the tuple $\langle x_s, x_f \rangle$. This MG is subsequently fed into the LLM under test to yield the corresponding outputs $f(x_s)$ and $f(x_f)$. A \textit{metamorphic oracle violation} occurs if and only if the inputs satisfy $\mathcal{R}_i$ while the corresponding outputs fail to satisfy $\mathcal{R}_o$. Consequently, a logical reasoning defect is detected when the following Boolean expression evaluates to \text{True}:
$$\mathcal{R}_i(x_s, x_f) \wedge \neg \mathcal{R}_o(f(x_s), f(x_f))$$

To illustrate this workflow, consider evaluating the logical reasoning capability of an LLM. Suppose we sample a source test case $x_s$ formulated as a natural language query: ``\textit{Premise: All birds can fly. Tweety is a bird. Conclusion: Tweety can fly. Is this valid?}''. Due to the absence of a test oracle, the correctness of the LLM's output $f(x_s)$ is unverifiable. However, by designing an MR based on the logical \textit{Law of Contraposition}, we can transform $x_s$ into a follow-up test case $x_f$: ``\textit{Premise: Anything that cannot fly is not a bird. Tweety is a bird. Conclusion: Tweety can fly. Is this valid?}''. Given the strict logical equivalence guaranteed by the Law of Contraposition under FOL, the two queries pose the exact same reasoning task. Therefore, the expected output relation $\mathcal{R}_o$ requires that the LLM must yield semantically equivalent responses for both queries (e.g., affirming the validity of both). If the LLM predicts $f(x_s)$ to be valid but $f(x_f)$ to be invalid, a metamorphic oracle violation is triggered, capturing a genuine logical reasoning defect in the model.
This example demonstrates how logic-based MT can expose internal reasoning inconsistencies without requiring ground truth. To provide a clear synthesis of how these elements align with the formal definitions of MT, we summarize the mapping between our formal notations and this concrete reasoning instance in Table~\ref{tab:mt_mapping}. 

\begin{table*}[pos=htbp, width=\textwidth]
\caption{Mapping of MT Component to the Logical Reasoning Case Study.}
\label{tab:mt_mapping}
\begin{tabular}{lcp{11cm}}
\toprule
\textbf{MT Component} & \textbf{Notation} & \textbf{Instantiation in Logical Reasoning Example} \\ \midrule
Source Test Case & $x_s$ & ``All birds can fly. Tweety is a bird. Conclusion: Tweety can fly.'' \\
Follow-up Test Case & $x_f$ & ``Anything that cannot fly is not a bird. Tweety is a bird. Conclusion: Tweety can fly.'' \\
Input Relation & $\mathcal{R}_i$ & Logical equivalence via the \textit{Law of Contraposition}. \\
Output Relation & $\mathcal{R}_o$ & Semantic consistency between the LLM's responses $f(x_s)$ and $f(x_f)$. \\
Metamorphic Relation & $\langle \mathcal{R}_i, \mathcal{R}_o \rangle$ & The constraint that queries transformed via the Law of Contraposition must yield semantically consistent reasoning results with the source queries. \\
Metamorphic Oracle Violation & $\mathcal{R}_i \wedge \neg \mathcal{R}_o$ & The LLM confirms the validity for $x_s$ but rejects it for $x_f$ (or vice versa). \\ \bottomrule
\end{tabular}
\end{table*}

\subsection{Motivating Example}

In this section, we present concrete examples to illustrate the inherent limitations of current LLM evaluation paradigms. First, we highlight why MT is essential for exposing logically inconsistent reasoning that standard reference-based testing obscures. Second, we demonstrate why we design MRs based on FOL rather than lexical transformations to prevent semantic drift.

\subsubsection{Reference-based Testing vs. MT}
Consider the logical entailment test case illustrated in Figure~\ref{fig:motivationExample1}. In the first query, the model is presented with a set of premises and a conclusion. The model answers \emph{``Unknown''}, correctly aligning with the ground-truth label to indicate it cannot derive the conclusion from the given premises. In the second query, we modify only a single premise by applying an equivalence transformation. Specifically, the rule stating that \emph{``all social media applications have chat features or video features''} is rewritten using implication elimination. All other premises, the conclusion, and the prompt format remain unchanged. 

\begin{figure}[pos=htbp]
    \centering
    \includegraphics[width=\columnwidth]{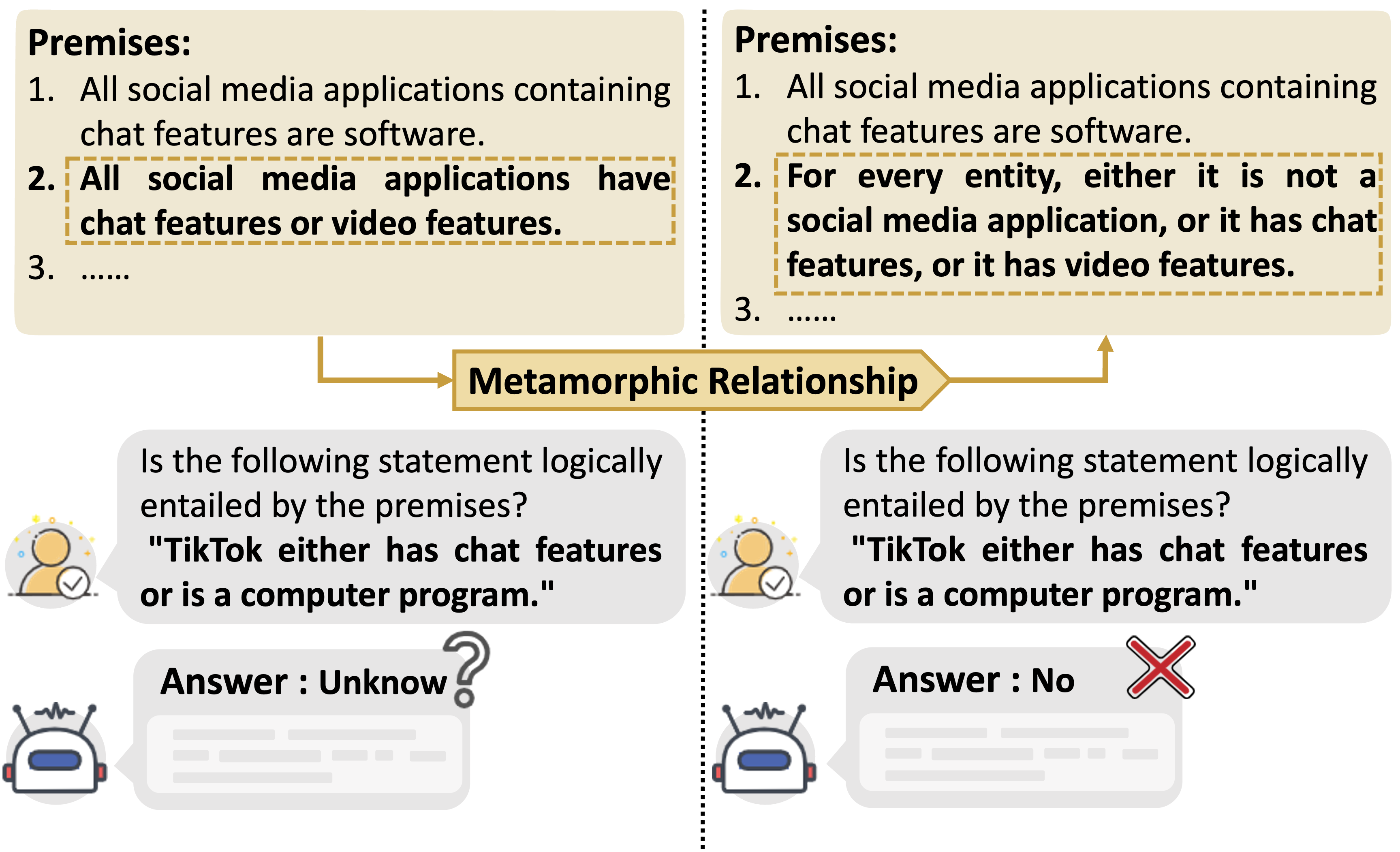}
    \caption{An illustration of an LLM's failure to maintain consistency across logically equivalent reasoning tasks.}
    \label{fig:motivationExample1}
\end{figure}

Under FOL, the transformed premise is logically equivalent to the original one. Therefore, the entailment relationship must remain unchanged. However, the model returns a logically inconsistent answer: \emph{``No,''} indicating a clear failure to maintain consistency across equivalent logical reasoning tasks.

Traditional reference-based testing evaluates these queries in isolation against a ground-truth label, scoring them merely as ``correct'' or ``incorrect.'' Such an evaluation obscures a critical issue: a single ``correct'' answer---such as the one achieved in the first query---may merely result from random guessing or superficial pattern matching rather than genuine logical reasoning. As demonstrated, one isolated success may provide a false sense of reasoning capability, failing to detect the underlying logical inconsistencies.

\subsubsection{Logical vs. lexical MRs}
Figure~\ref{fig:motivationExample2} demonstrates how lexical MRs can induce semantic drift. First, we establish a baseline where the model correctly deduces the conclusion (\emph{``I own a car.''}) from a specific premise (\emph{``My car has four wheels.''}), returning an accurate judgment of \emph{``Yes''}. Then, the premise is modified by replacing the target noun \emph{``car''} with its broader hypernym, \emph{``vehicle,''} while keeping the conclusion unchanged.

\begin{figure}[pos=htbp]
    \centering
    \includegraphics[width=\columnwidth]{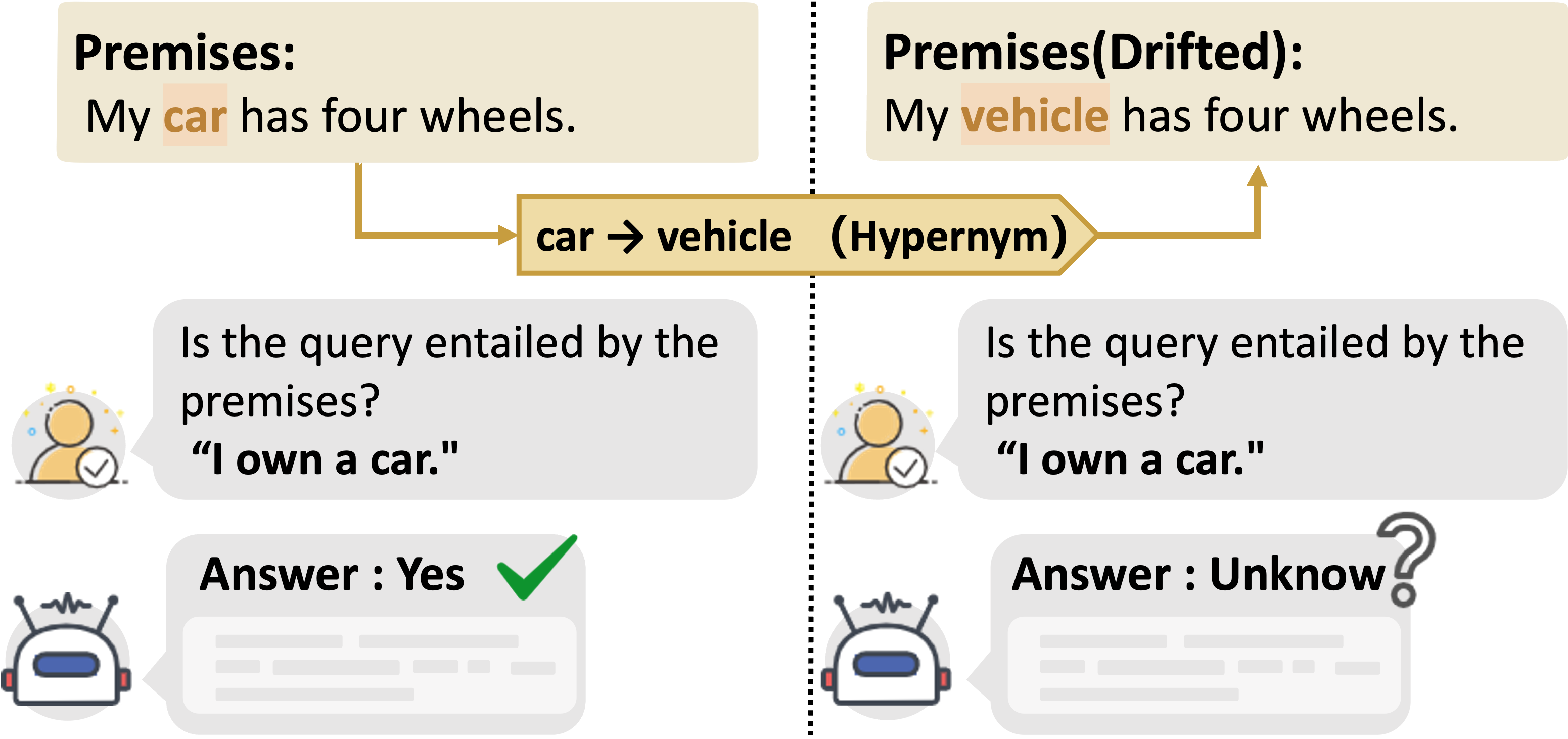}
    \caption{An illustration of semantic drift caused by traditional lexical MRs.}
    \label{fig:motivationExample2}
\end{figure}

This minor substitution induces a semantic drift that invalidates the original deductive chain. Because the category of \emph{``vehicle''} has a broader scope, a \emph{``four-wheeled vehicle''} is not necessarily a car---it could realistically be a quadricycle or an ATV. Consequently, the modified premise no longer strictly entails the conclusion. Recognizing this logical difference, the LLM correctly updates its prediction to \emph{``Unknown.''}

\revise{To avoid this type of semantic drift, logical MRs change the transformation target. Instead of directly editing words in natural language, LGMT first transforms the underlying FOL representation using equivalence-preserving laws and then translates the transformed formula into natural language. Therefore, semantic preservation is guaranteed at the symbolic level, and the remaining risk is limited to the realization step from FOL to natural language, which we audit empirically.}

\section{Methodology}
\label{sec:method}
In this section, we present \textbf{LGMT}, an automated framework designed to validate the logical reasoning reliability of LLMs. 
We first outline the testing workflow, then detail the core logic-grounded MRs, and finally describe the pipeline for constructing natural language test cases.

\subsection{Framework Overview}
\label{subsec:framework_overview}
Figure~\ref{fig:framework} illustrates the overall workflow of the proposed LGMT framework. LGMT is designed as a fully automated, oracle-free testing pipeline for evaluating the logical reasoning reliability of LLMs through logic-grounded metamorphic transformations and consistency checking.
\revise{In this workflow, LGMT follows the oracle-free nature of MT: instead of requiring a ground-truth oracle for each generated follow-up test case, it checks whether model predictions satisfy the expected metamorphic relation. Accordingly, the detection of metamorphic oracle violations does not depend on benchmark labels, although such labels are retained for comparative experimental metrics such as static accuracy, consistent accuracy, HDR, and FUR.}

\begin{figure*}
    \centering
    \includegraphics[width=\textwidth]{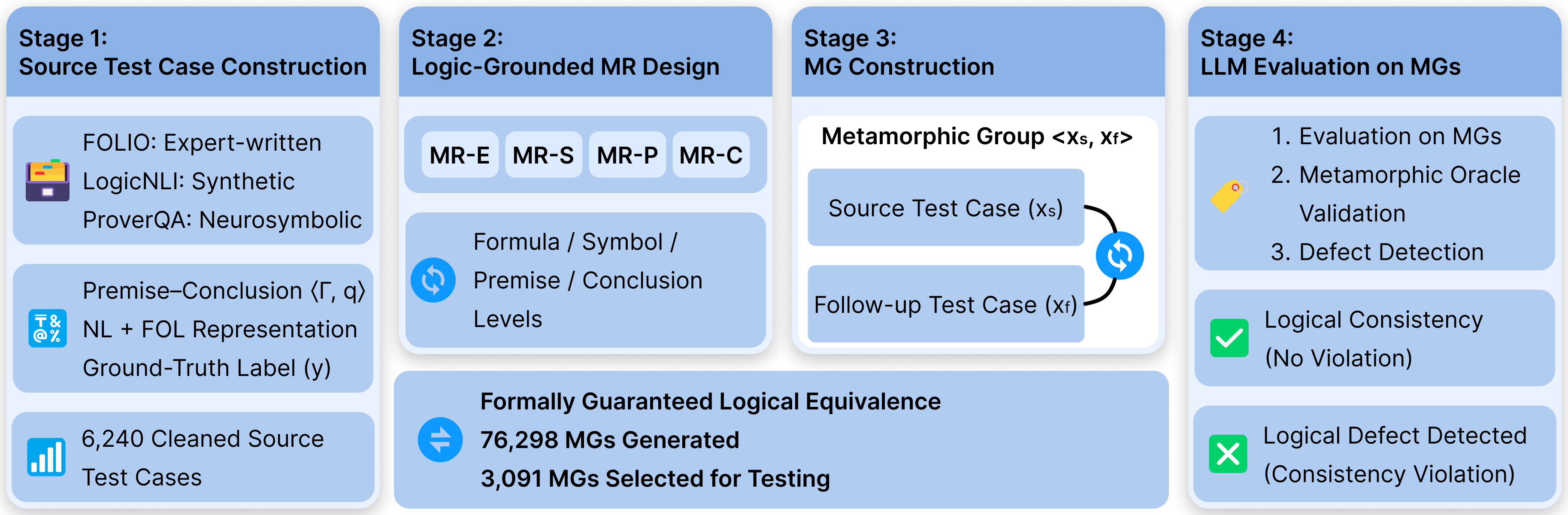}
    \caption{Overview of the LGMT Framework.}
    \label{fig:framework}
\end{figure*}

The framework is organized into four sequential stages, aligned with the methodological components described in the following subsections.

\textbf{(1) Source Test Case Construction.}
We construct source test cases from multiple FOL reasoning benchmarks (e.g., FOLIO, LogicNLI, and ProverQA). Each test case is represented in both natural language (premises and conclusion) and its corresponding FOL form, which serves as the basis for subsequent transformations.

\textbf{(2) Logic-Grounded MR Design.}
We define a set of MRs grounded in first-order logic, including formula-level (MR-E), symbol-level (MR-S), premise-level (MR-P), and conclusion-level (MR-C) transformations. These MRs are formally guaranteed to preserve logical equivalence.

\textbf{(3) Metamorphic Group Construction.}
Given a source test case, the selected MR is applied to its FOL representation to generate a transformed logical structure, which is then translated back into natural language. The source and transformed test cases together form a metamorphic group.

\textbf{(4) LLM Evaluation and Defect Detection.}
Each MG is executed on the LLM under test. The model outputs are parsed into structured predictions and evaluated using a metamorphic oracle, which checks whether the expected consistency relation is preserved. Violations indicate logical reasoning defects.

Through this pipeline, LGMT shifts evaluation from isolated correctness to consistency under logic-preserving transformations, enabling scalable and oracle-free detection of reasoning failures. For clarity, we provide an end-to-end example of this pipeline in Appendix~\ref{app:example}.

\subsection{Source Test Case Construction}
\label{subsec:stc_construction}

Each source test case for logical reasoning consists of a set of premises and a target conclusion. To support logic-grounded transformations, each source test case is represented in both natural language and its corresponding FOL form.

The construction of source test cases follows two key principles. First, the natural-language formulation must admit a faithful symbolic representation, ensuring that its logical structure can be accurately captured in FOL. Second, the symbolic representation must support the application of logic-preserving metamorphic relations, enabling systematic generation of follow-up test cases without altering the underlying semantics.

Accordingly, each source test case is normalized into a unified representation $\langle \Gamma, q, y \rangle$, where $\Gamma$ denotes a finite set of premises, $q$ is the target conclusion, and $y \in \{\text{True}, \text{False}, \text{Unknown}\}$ is the ground-truth label when available. While LGMT does not rely on ground-truth labels for defect detection, these labels are retained exclusively for quantitative evaluation and comparative analysis in the experimental study.

\subsection{Logic-Grounded MRs}
\label{subsec:mrs}
To address the semantic drift induced by traditional lexical MT, we formulate 20 types of MRs strictly grounded in first-order logic.
As summarized in Table~\ref{tab:mr_categories}, we systematically organize these MRs into four distinct categories based on their \textit{transformation granularity}: from atomic symbols up to the structure of the test case.

\revise{It is important to note that MR-E1 underlies the canonical normalization used in the PNNF completeness proof (as formally proved in Appendix~\ref{app:pnnf-completeness}). The completeness claim here is \emph{not} semantic completeness of first-order logic. Rather, it is completeness with respect to the PNNF-oriented normalization target used in this work: under the proposed rewrite procedure, any FOL formula can be rewritten into a canonical PNNF normal form using the atomic rewrite rules defined in MR-E1, and the obtained normal form is unique up to $\alpha$-equivalence. In the context of LLM evaluation, this means that MR-E1 captures a core set of equivalence-preserving transformations required for systematic logical normalization.}

\revise{We conduct the formal completeness analysis only for MR-E1 because MR-E1 is designed as the normalization-oriented core rewrite system. By contrast, MR-E2 contains additional valid equivalence laws outside this PNNF normalization basis, while MR-S, MR-P, and MR-C operate at the symbol, premise, and conclusion levels of a test case. These MRs are logic-preserving and useful for testing, but they are not intended to constitute a complete normalization procedure.}

\begin{table*}[pos=htbp, width=\textwidth]
    \centering
    \caption{Logic-Grounded MRs}
    \label{tab:mr_categories}
    \begin{tabular}{lc c l p{0.45\textwidth}}
        \toprule[1pt]
        \textbf{Category} & \textbf{Subclass} & \textbf{\# of MRs} & \textbf{Granularity} & \textbf{Transformation Strategy} \\ 
        \midrule
        \multirow{2}{*}{\textbf{MR-E}} & MR-E1 & 6 & Formula level &  Constitutes a complete basis for PNNF (Appendix~\ref{app:pnnf-completeness}). \\
        \cmidrule{2-5}
                                       & MR-E2 & 4 & Formula level & Captures other valid logical equivalence transformations outside PNNF. \\ 
        \midrule
        \textbf{MR-S}     &   --   & 2 & Symbol level & Uniformly renames     constants or predicates. \\ 
        \midrule
        \textbf{MR-P}     &   --  & 5 & Premise level & Structurally modifies $\Gamma$. \\ 
        \midrule
        \textbf{MR-C}     &   --  & 3 & Conclusion level & Structurally modifies $q$. \\ 
        \bottomrule[1pt]
    \end{tabular}
\end{table*}

As established in Section~\ref{subsec:stc_construction}, the logical skeleton of a test case is structured as a premise--conclusion pair $(\Gamma, q)$. Throughout the subsequent formalizations, we use $\phi$, $\psi$, and $\theta$ to represent arbitrary first-order formulas involved in our MRs.
An MR transforms a source test case $(\Gamma, q)$ into a follow-up test case $(\Gamma', q')$. We use the symbol $\leadsto$ to denote an application of an MR. Specifically, at the formula level, $\phi \leadsto \psi$ denotes a one-step rewrite of formulas induced by an MR; at the test-case level, $(\Gamma, q) \leadsto (\Gamma', q')$ denotes the corresponding transformation between test cases.

\subsubsection{MR-E: Logical Equivalence Transformations}
\label{subsubsec:mr_e}

This group of MRs defines a set of logically equivalent rewrite rules for FOL formulas. Each MR is formulated as a rule schema that applies to arbitrary subformulas during logical normalization. As summarized in Table~\ref{tab:mr_categories}, this category comprises two subsets: MR-E1, which forms the mathematical basis of the PNNF completeness result (Theorem~\ref{the:pnnf-complete}), and MR-E2, which captures additional valid equivalence rules not included in MR-E1.

\textbf{MR-E1: Canonical Equivalence-Based MRs.}
This subset forms the mathematical basis of the PNNF completeness result (Theorem~\ref{the:pnnf-complete}). It ensures that any formula can be reduced to a standard form:
\vspace{1ex}

\noindent\textit{MR-E1.1 (Implication and biconditional elimination).}
Implication and biconditional connectives are eliminated using the following rewrite rules:
$$ \varphi \rightarrow \psi \leadsto \neg \varphi \lor \psi $$
$$ \varphi \leftrightarrow \psi \leadsto (\neg \varphi \lor \psi) \land (\neg \psi \lor \varphi) $$

\noindent\textit{MR-E1.2 (Negation normalization).}
Negation is pushed inward until it applies only to atomic formulas through double negation, De Morgan's laws, and quantifier duality:
$$ \neg\neg\varphi \leadsto \varphi $$
$$ \neg(\varphi \land \psi) \leadsto (\neg\varphi) \lor (\neg\psi), \quad \neg(\varphi \lor \psi) \leadsto (\neg\varphi) \land (\neg\psi) $$
$$ \neg\forall x\,\varphi \leadsto \exists x\,\neg\varphi, \quad \neg\exists x\,\varphi \leadsto \forall x\,\neg\varphi $$

\noindent\textit{MR-E1.3 (Quantifier lifting).}
To move quantifiers toward the prenex prefix, let $Q \in \{\forall,\exists\}$, $\circ \in \{\wedge,\vee\}$ and $\mathrm{FV}(\psi)$ to denote the free variables in $\psi$. Provided $x \notin \mathrm{FV}(\psi)$, the following lifting rules apply:
$$ (Qx\,\varphi)\circ\psi \leadsto Qx\,(\varphi\circ\psi) $$
$$ \psi\circ(Qx\,\varphi) \leadsto Qx\,(\psi\circ\varphi) $$

\noindent\textit{MR-E1.4 (Structural normalization).}
Conjunctions and disjunctions are flattened and reordered according to a fixed canonical ordering $\prec$ (e.g., lexicographic order):
$$ (\varphi \land (\psi \land \theta)) \leadsto (\varphi \land \psi \land \theta) $$
$$ (\varphi \lor (\psi \lor \theta)) \leadsto (\varphi \lor \psi \lor \theta) $$
$$ (\varphi_1 \land \cdots \land \varphi_n) \leadsto (\varphi_{\sigma(1)} \land \cdots \land \varphi_{\sigma(n)}) $$
where $\sigma$ is the permutation that sorts the sequence by $\prec$.
\vspace{1ex}

\noindent\textit{MR-E1.5 ($\alpha$-renaming).}
To resolve variable name conflicts that hinder quantifier lifting, we perform $\alpha$-conversion. Let $\mathrm{Var}(\varphi)$ denote the set of all variables in $\varphi$:
$$ (Qx\,\varphi)\circ\psi \leadsto (Qy\,\varphi[x:=y])\circ\psi $$
where $x \in \mathrm{Var}(\psi)$ and $y \notin \mathrm{Var}(\varphi) \cup \mathrm{Var}(\psi)$.

\noindent\textit{MR-E1.6 (Canonical quantifier ordering).}
Within a prenex prefix, adjacent identical quantifiers (blocks) are reordered according to the first occurrence of their variables in the matrix $M$:
$$ Qx_1 \cdots Qx_r . M \leadsto Qx_{\sigma(1)} \cdots Qx_{\sigma(r)} . M $$
where $\sigma$ ensures a canonical variable sequence within each quantifier block.

\textbf{MR-E2: Non-Canonical Equivalence-Based MRs.}
While MR-E1 covers the complete set for canonical normalization, MR-E2 captures additional valid equivalence laws that further challenge the model's ability to simplify or expand logical expressions.

\vspace{1ex}
\noindent\textit{MR-E2.1 (Redundancy elimination).}
This type of MR leverages idempotence and absorption laws to eliminate duplicated or subsumed subformulas:
$$ \varphi \lor \varphi \leadsto \varphi, \quad \varphi \land \varphi \leadsto \varphi $$
$$ \varphi \lor (\varphi \land \psi) \leadsto \varphi, \quad \varphi \land (\varphi \lor \psi) \leadsto \varphi $$

\noindent\textit{MR-E2.2 (Tautology and contradiction elimination).}
This type of MR eliminates trivial formulas using the laws of excluded middle and contradiction:
$$ \varphi \lor \neg \varphi \leadsto \text{True},  \quad  \varphi \land \neg \varphi \leadsto \text{False}$$

\noindent\textit{MR-E2.3 (Identity and domination laws).}
This type of MR evaluates the model's capacity to prune logically inert or dominant components:
$$ \varphi \land \text{True} \leadsto \varphi, \quad \varphi \lor \text{False} \leadsto \varphi $$
$$ \varphi \lor \text{True} \leadsto \text{True}, \quad \varphi \land \text{False} \leadsto \text{False} $$

\noindent\textit{MR-E2.4 (Distributivity law).}
This MR rewrites formulas by redistributing conjunctions over disjunctions (and vice versa):
$$ \varphi \land (\psi \lor \theta) \leadsto (\varphi \land \psi) \lor (\varphi \land \theta) $$
$$ \varphi \lor (\psi \land \theta) \leadsto (\varphi \lor \psi) \land (\varphi \lor \theta) $$

\subsubsection{MR-S: Symbol Renaming Invariance}
This group of MRs captures invariance under systematic renaming of non-logical symbols, including constant symbols and predicate symbols.
Such renaming operations preserve the logical form of formulas and the entailment relations between premises and conclusions, provided that the renaming is applied uniformly.
Therefore, a logically sound reasoner is expected to produce identical entailment judgments before and after symbol renaming.
We use the notation $(\Gamma, q)[s := s']$ to denote the test case obtained by uniformly replacing every occurrence of the symbol $s$ in both $\Gamma$ and $q$ with the symbol $s'$.

\noindent\textit{MR-S1 (Constant Renaming Invariance).}
This type of MR captures invariance under uniform renaming of constant symbols, which correspond to object identifiers in the domain.
A constant symbol is consistently replaced by a fresh constant symbol throughout the test case, while all other symbols remain unchanged.
At the test-case level, the MR is defined as:

$$(\Gamma, q) \leadsto (\Gamma, q)[c := c']$$
where $c$ is a constant symbol occurring in $(\Gamma, q)$ and $c'$ is a fresh constant symbol.

\noindent\textit{MR-S2 (Predicate Renaming Invariance).}
This type of MR captures invariance under uniform renaming of predicate symbols, which label relations over domain objects.
A predicate symbol is consistently replaced by a fresh predicate symbol throughout the test case, while all other symbols remain unchanged.
At the test-case level, the MR is defined as:
$$ (\Gamma, q) \leadsto (\Gamma, q)[P := P'] $$
where $P$ is a predicate symbol occurring in $(\Gamma, q)$ and $P'$ is a fresh predicate symbol.

\subsubsection{MR-P: Premise Transformation}
This group of MRs captures the invariance of entailment judgments under premise transformations which are inspired by Chen et al.~\cite{chen2021testing}. At the test-case level, these MRs are defined as:
$$ (\Gamma, q) \leadsto (\Gamma', q) $$
where $\Gamma'$ is obtained from $\Gamma$ by applying the corresponding premise transformation.

\noindent\textit{MR-P1 (Premise Reordering).}
The order of premises in $\Gamma$ is permuted.
This MR changes only the presentation order of premises and does not affect their logical content.

\noindent\textit{MR-P2 (Premise Duplication).}
A premise already occurring in $\Gamma$ is duplicated and added again to the premise set.
This introduces redundancy without adding new information.

\noindent\textit{MR-P3 (Irrelevant Extension).}
A logically irrelevant but non-contradictory premise is added to the premise set.
Such a premise does not affect the entailment relation between $\Gamma$ and $q$.

\noindent\textit{MR-P4 (Premise Fusion).}
Two premises $\varphi, \psi \in \Gamma$ are combined into a conjunctive premise $\varphi \land \psi$, which is added to the premise set.
This transformation preserves information while altering the structural form of the premises.

\noindent\textit{MR-P5 (Premise Decomposition).}
A conjunctive premise $\varphi \land \psi \in \Gamma$ is decomposed into two separate premises $\varphi$ and $\psi$, which are added to the premise set.

\subsubsection{MR-C: Conclusion Transformation}
This group of MRs captures invariance of entailment judgments under structural transformations of the conclusion.
At the test-case level, this group of MRs is defined as:
$$ (\Gamma, q) \leadsto (\Gamma, q') $$
where $q'$ is obtained from $q$ by applying the corresponding conclusion transformation.

\noindent\textit{MR-C1 (Conjunction with Truth).}
The conclusion $q$ is transformed into $q \land \textrm{True}$.

\noindent\textit{MR-C2 (Disjunction with Falsehood).}
The conclusion $q$ is transformed into $q \lor \textrm{False}$.

\noindent\textit{MR-C3 (Double Negation).}
The conclusion $q$ is transformed into $\neg\neg q$.

\revise{To make these MRs easier to apply in practice, Section~\ref{subsec:mr_application_example} provides a compact example showing how the four categories of MRs are applied in LLM testing.}

\subsection{Metamorphic Group Construction}
\label{subsec:mg_construction}

Building upon the MRs defined in Section~\ref{subsec:mrs}, we apply these logical transformations to each source test case to generate its corresponding follow-up test case ($x_s \leadsto x_f$). As established in Section~\ref{subsec:stc_construction}, our framework maintains a dual-representation system. Under this design, the application of an MR mutates the underlying FOL skeleton first (e.g., $\phi \leadsto \psi$, where $\phi \in \Gamma$). 

To translate this mutated FOL formula back into natural language, we employ an LLM-based translation module. Crucially, the LLM is constrained via prompting to operate strictly as a parser and translator rather than a reasoner. This constrains the translation module to preserve the intended logical structure as much as possible, while avoiding additional reasoning or simplification. Finally, the source test case and its generated follow-up test case are coupled to form a complete metamorphic group for evaluation, formally defined as the tuple $\langle x_s, x_f \rangle$.


\subsection{Metamorphic Oracle and Defect Detection}

The final module of the LGMT framework is automated test result verification. Given an MG $\langle x_s, x_f \rangle$, we systematically evaluate the reasoning robustness of the LLM under test, denoted as $f(\cdot)$. We feed $x_s$ and $x_f$ independently into the model to obtain their respective raw textual outputs, $f(x_s)$ and $f(x_f)$.

Because LLM outputs are generative, we parse these raw responses to extract discrete logical predictions (e.g., True, False, or Unknown), denoted as $\hat{y}_s$ and $\hat{y}_f$.
The metamorphic oracle then acts as an automated, black-box verification mechanism by checking whether the relationship between $\hat{y}_s$ and $\hat{y}_f$ satisfies the expected output relation ($\mathcal{R}_o$) dictated by the applied MR.

Any deviation from the expected $\mathcal{R}_o$ constitutes a \textit{metamorphic oracle violation}. Because our FOL-grounded MRs strictly preserve the underlying logic of source test cases, each violation indicates a potential reasoning inconsistency, after excluding translation and parsing artifacts.

\subsection{Illustrative MR Applications}
\label{subsec:mr_application_example}

\revise{Following the MR categorization by transformation granularity in Section~\ref{subsec:mrs} and Table~\ref{tab:mr_categories}, we provide an example to bridge the formal MR definitions and their practical use in LLM testing. To keep the discussion connected to the motivating example in Figure~\ref{fig:motivationExample1}, consider the following simplified source test case:
\begin{quote}
\textbf{Premises:} (1) All social media applications have chat features or video features. (2) TikTok is a social media application.\\
\textbf{Conclusion:} TikTok has chat features or video features.
\end{quote}
Its FOL representation is $\Gamma=[\forall x(S(x)\rightarrow (C(x)\lor V(x))), S(t)]$ and $q=C(t)\lor V(t)$, where $S$ denotes \emph{social media application}, $C$ denotes \emph{has chat features}, $V$ denotes \emph{has video features}, and $t$ denotes TikTok. The expected judgment is \textit{True}. LGMT does not need this label to detect a metamorphic violation, and the label is shown here only to make the example intuitive.}

\revise{\textbf{Formula-level application.} At the formula level, LGMT applies an equivalence rule to a formula in the test case. For example, MR-E1.1 rewrites the first premise by implication elimination:
\[
\forall x(S(x)\rightarrow (C(x)\lor V(x)))
\leadsto
\forall x(\neg S(x)\lor C(x)\lor V(x)).
\]
The corresponding follow-up case uses the premise \emph{``For every application, either it is not a social media application, or it has chat features or video features''}, while keeping the second premise and the conclusion unchanged. Because the transformed premise is logically equivalent to the original one, a reliable reasoner should return the same label for the source and follow-up cases.}

\revise{\textbf{Symbol-level application.} At the symbol level, LGMT uniformly renames non-logical symbols while preserving the inference structure. For example, MR-S can replace $S$, $C$, and $V$ with fresh predicates $P$, $Q$, and $R$, and replace the constant $t$ with $a$, throughout both $\Gamma$ and $q$. A follow-up case can therefore be realized as: \emph{``All objects with property P have property Q or property R. a has property P. Conclusion: a has property Q or property R.''} The expected judgment must remain unchanged.}

\revise{\textbf{Premise-level application.} At the premise level, LGMT modifies the organization of the premise set without changing its logical content. For example, MR-P1 can reorder the two premises, and MR-P2 can duplicate the premise \emph{``TikTok is a social media application.''} In both cases, the available information is unchanged, so the model should preserve its original judgment.}

\revise{\textbf{Conclusion-level application.} At the conclusion level, LGMT applies an equivalence-preserving transformation to $q$. For example, MR-C2 rewrites the conclusion as $(C(t)\lor V(t))\lor \mathrm{False}$, which can be realized as \emph{``TikTok has chat features or video features, or five times five equals one.''} The follow-up conclusion is logically equivalent to the original conclusion, so the model's entailment judgment should remain the same.}

\revise{This example illustrates the role of the four MR categories in LGMT. Formula-level MRs alter logical expressions, symbol-level MRs abstract away lexical cues, premise-level MRs modify the structure of the premise set, and conclusion-level MRs reformulate the query target. In all cases, the expected output relation is consistency between the source and follow-up predictions.}

\section{Experimental Setup}
\label{sec:experimental_setup}
In this section, we present the experimental design for evaluating the effectiveness of our LGMT framework. We first outline the core research questions guiding our empirical study, then detail the selection of models under test, the test datasets and sampling strategy, and finally describe the evaluation metrics and analysis methods.

\subsection{Research Questions}
\label{subsec:rqs}
To systematically and rigorously evaluate our framework, we organize our empirical study around four core research questions, inspired by recent state-of-the-art evaluation paradigms for MT for LLMs~\cite{cho2025metamorphic}:

\begin{itemize}
    \item \textbf{RQ1 (Effectiveness):} How effectively can LGMT detect logical reasoning defects in LLMs?
    
    \item \textbf{RQ2 (Comparison):} How does LGMT compare with traditional reference-based testing?
    
    \item \textbf{RQ3 (Sensitivity):} How do different categories of logic-grounded MRs impact the inconsistency detection rate?
    
    \item \textbf{RQ4 (Mitigability):} Can advanced prompting strategies, such as Few-shot and Few-shot CoT prompting, mitigate the logical defects exposed by LGMT?
\end{itemize}

These RQs are designed to comprehensively validate LGMT from both testing and evaluation perspectives. \textbf{RQ1} assesses the overall fault-detection effectiveness of LGMT (i.e., metamorphic oracle violation rate). \textbf{RQ2} compares our metamorphic oracle against traditional reference-based testing methods to expose hidden defects that successfully evade traditional reference-based evaluations. \textbf{RQ3} provides a sensitivity analysis across different categories of our logic-grounded MRs. Finally, \textbf{RQ4} investigates whether advanced prompting strategies (e.g., \textit{Zero-shot}, \textit{Zero-shot CoT}, \textit{Few-shot}, and \textit{Few-shot CoT}) can effectively mitigate the exposed logical defects.

\subsection{Models Under Test}
\label{subsec:suts}

To systematically evaluate the generalizability of our LGMT framework, we carefully selected six state-of-the-art LLMs. As shown in Table~\ref{tab:suts}, these models are categorized into three distinct categories: \textit{Proprietary Models}, \textit{Open-Weights Models}, and \textit{Reasoning Models} (which are specifically optimized for complex reasoning tasks). Proprietary and reasoning models were accessed via their official API providers, whereas standard open-weights models were hosted on high-performance inference platforms (e.g., Fireworks AI).

To ensure a fair and deterministic evaluation across these diverse hosting environments, we strictly set \texttt{temperature=0} (greedy decoding) for all models. This eliminates the inherent stochasticity of LLM generation, ensuring that the detected metamorphic oracle violations reflect structural reasoning defects rather than random sampling variations. 
The only exception is for reasoning models, where the temperature parameter is inherently restricted by the official API design.

\begin{table}[pos=htbp]
    \centering
    \caption{Summary of the Evaluated Models}
    \label{tab:suts}
    \begin{tabular}{p{0.2\columnwidth} p{0.32\columnwidth} p{0.3\columnwidth}}
        \toprule
        \textbf{Category} & \textbf{Model Name} & \textbf{API Identifier} \\
            
        \multirow{2}{*}{Proprietary} 
        & GPT-5.2 & \texttt{gpt-5.2} \\
        & Claude 4.5 Sonnet & \texttt{claude-4-5-sonnet} \\
        \midrule
        \multirow{2}{*}{Open-Weights} 
        & Llama 3.3 (70B) & \texttt{Llama-3.3-70B-Instruct} \\
        & DeepSeek Chat & \texttt{deepseek-chat} \\
        \midrule
        \multirow{2}{*}{Reasoning} 
        & OpenAI o4-mini & \texttt{o4-mini} \\
        & DeepSeek Reasoner & \texttt{deepseek-reasoner} \\
        \bottomrule
    \end{tabular}
\end{table}

\subsection{Test Datasets and Sampling Strategy} 
\label{subsec:test_datasets} 

To instantiate the source test case construction process described in Section~\ref{subsec:stc_construction}, we derive source test cases from three diverse FOL reasoning benchmarks: FOLIO~\cite{han2024folio}, LogicNLI~\cite{tian2021diagnosing}, and ProverQA~\cite{qi2024large}. To ensure the quality of our test cases, we applied rigorous data cleaning and filtration criteria tailored to each dataset. Specifically, for FOLIO, we extracted 995 instances and filtered out those with unparseable logical expressions. For LogicNLI, we developed a customized logical solver to identify and eliminate instances with inherently contradictory premises, ultimately yielding 4,245 high-quality source test cases. For ProverQA, we utilized the complete set of 1,000 instances. By applying our 20 types of logic-grounded MRs to these source test cases, our LGMT framework automatically generated \revise{a large candidate pool} of 76,298 MGs. 

However, executing all 76,298 MGs across multiple LLMs is computationally prohibitive. To balance computational feasibility with strict statistical guarantees, we employed a sampling strategy. 
According to Cochran's formula~\cite{cochran1968errors}, a sample size of $N \ge 385$ ensures a margin of error within $5\%$ at a $95\%$ confidence level. To satisfy this threshold for every major analytical category (MR-E, MR-S, MR-P, and MR-C) while ensuring fine-grained diversity, we randomly sampled up to 200 instances per sub-rule (e.g., MR-S1, MR-S2). Because each major category comprises at least two sub-rules, this strategy naturally aggregates to a minimum of 400 instances per category (e.g., MR-S yields $2 \times 200 = 400$). Finally, our evaluation test dataset constitutes a total of 3,091 MGs. \revise{The distribution of sampled MGs across the four MR categories is shown in Table~\ref{tab:sample_distribution}.}

\begingroup
\begin{table}[pos=htbp]
\centering
\caption{Distribution of sampled MGs across MR categories}
\label{tab:sample_distribution}
\begin{tabular}{lcc}
\toprule
\textbf{MR Category} & \textbf{Sampled MGs} & \textbf{Percentage} \\
\midrule
MR-E & 1,091 & 35.30\% \\
MR-S & 400 & 12.94\% \\
MR-P & 1,000 & 32.35\% \\
MR-C & 600 & 19.41\% \\
\midrule
Total & 3,091 & 100.00\% \\
\bottomrule
\end{tabular}
\end{table}
\endgroup

It is important to note that due to strict structural applicability constraints, certain MRs yielded fewer than the 200-instance maximum. Nevertheless, we manually confirmed that the aggregated sample size for every major category safely remains well above the 385 threshold (specifically, the smallest category comprises exactly 400 cases).

\subsection{Prompting Strategies} 
When evaluating LLMs, a common concern is that test failures might simply result from inadequate prompting rather than genuine reasoning defects. To address this (RQ4), we evaluate the models under four established paradigms: Zero-shot, Zero-shot CoT, Few-shot, and Few-shot CoT. While \textit{Zero-shot CoT} serves as our default setting (RQ1--RQ3), contrasting it with the other paradigms allows us to investigate whether Few-shot and CoT prompting strategies can effectively mitigate these logical defects. 

\subsection{Evaluation Metrics} \label{subsec:metrics}
To systematically address our Research Questions, we define three categories of evaluation metrics, precisely aligned with our testing objectives and comparative paradigms.

\subsubsection{Metrics for Fault Detection Effectiveness (RQ1, RQ3, RQ4)}
To evaluate the absolute capability of LGMT in exposing logical reasoning defects without relying on ground-truth labels, we employ
\textit{Metamorphic Oracle Violation Rate (MVR)}, which is defined as the proportion of metamorphic groups where the LLM's outputs violate the corresponding metamorphic relation~\cite{cho2025metamorphic}. MVR directly quantifies the inconsistency detection rate of LGMT in exposing logical reasoning defects.

\subsubsection{Metrics for LGMT Reliability and Validity (RQ1)}
As established in Section \ref{sec:method}, the theoretical soundness of our testing framework is guaranteed by our FOL-grounded approach. To empirically validate this, we define \textit{False Report Rate (FRR)} as the proportion of reported metamorphic oracle violations that are actually false positives. To quantify this, we categorize false positives into two types based on their root causes.

The first type corresponds to violations caused by \textit{semantic drift}, where the natural-language realization of the follow-up test case deviates from its underlying FOL representation, resulting in a mismatch between the intended logical transformation and the actual input semantics.

The second type corresponds to violations caused by \textit{parsing errors}, where the LLM's raw output is incorrectly mapped to a discrete logical label. In these cases, the inconsistency arises from output misinterpretation rather than the model's reasoning behavior.
\revise{Specifically, all evaluation prompts require a strict JSON output with a \texttt{label} field. If the output cannot be parsed into one of the three valid labels, or if the parsed label is ambiguous, the case is logged with its raw output and MG identifier.}

Evaluated through rigorous manual inspection on a random sample of $N$ reported violations, FRR is calculated as:
$$FRR = \frac{FP_{drift} + FP_{parse}}{N}$$
where $FP_{drift}$ and $FP_{parse}$ denote the numbers of false positives caused by semantic drift and parsing errors, respectively, and $N$ is the total number of inspected reported violations.

\subsubsection{Metrics for Comparison and Analysis (RQ2, RQ4)}
To benchmark LGMT against traditional reference-based testing, we utilize ground-truth labels to define the following metrics:
\begin{itemize}
    \item \textbf{Static Accuracy ($Acc_{static}$):} The standard accuracy of the models on the source test cases, evaluated strictly against ground-truth labels. This metric captures the model's apparent reasoning capability under the traditional static reference-based testing.
    
    \item \textbf{Consistent Accuracy ($Acc_{cons}$):} The proportion of metamorphic test groups where the models correctly answer both the source test case and the follow-up test cases. This metric reduces the impact of coincidental correctness by requiring consistent predictions across logically equivalent test cases.
     
    \item \textbf{Hidden Defect Rate (HDR):} HDR is defined as the proportion of test cases that yield correct predictions under traditional static evaluation ($Acc_{static}$) but trigger a metamorphic oracle violation. This metric quantifies the performance degradation from $Acc_{static}$ to $Acc_{cons}$, reflecting the extent to which static benchmarks overestimate model capabilities due to hidden defects that evade traditional evaluation but are exposed by LGMT.
\end{itemize}

\subsubsection{Metrics for Undetected Logical Errors (RQ4)}

While LGMT detects logical inconsistencies through metamorphic oracle violations, it cannot capture cases where the model produces identical but incorrect predictions across both the source and follow-up test cases.
To quantify this limitation, we define the \textit{False Unreported Rate (FUR)} as the proportion of metamorphic groups where both answers are incorrect but no metamorphic oracle violation is triggered.
FUR characterizes a fundamental blind spot of consistency-based testing, where models may exhibit consistent but incorrect predictions. We note that FUR is not equivalent to the classical false negative rate, as it is defined with respect to benchmark ground-truth labels and only captures a subset of undetected errors.

\section{Experimental Results and Analysis}
\label{sec:results}
In this section, we present the empirical results and analysis of our LGMT framework to answer the four research questions formulated in Section~\ref{subsec:rqs}.

\subsection{Answer to RQ1}
\label{subsec:answer_rq1}
\setrq{1}

RQ1 investigates whether LGMT can effectively expose logical reasoning defects in LLMs while maintaining a low false report rate. To answer this question, we evaluate LGMT from two complementary perspectives: \textit{effectiveness}, measured by the Metamorphic Oracle Violation Rate (MVR), and \textit{reliability}, measured by the False Report Rate (FRR).

As shown in Table~\ref{tab:rq1}, LGMT detects a substantial number of metamorphic oracle violations across all evaluated models, with MVR ranging from 20.84\% to 29.44\%. DeepSeek Reasoner exhibits the highest violation rate (29.44\%), followed by GPT-5.2 (26.76\%), while Claude 4.5 Sonnet shows the lowest MVR (20.84\%).

\Finding{All evaluated models exhibit substantial logical reasoning defects, and the observed violation rates do not consistently correlate with model performance or reasoning-oriented design.}
\Implication{Although Claude 4.5 Sonnet achieves the lowest MVR (20.84\%), its violation rate remains non-trivial. GPT-5.2 still exhibits a relatively high MVR of 26.76\%, while the two reasoning-oriented models show markedly different robustness: o4-mini performs comparatively well (22.06\%), but DeepSeek Reasoner records the highest MVR overall (29.44\%). These results suggest that the observed violation rates do not consistently track overall model performance or reasoning-oriented design.}

\begin{table}[pos=htbp, width=0.9\textwidth]
\centering
\caption{Overall Effectiveness and Reliability of LGMT}
\label{tab:rq1}
\begin{tabular}{lcccccc}
\toprule
Model & MVR & \# Violations\\
\midrule
GPT-5.2 & 26.76\% & 827\\
Claude 4.5 Sonnet & 20.84\% & 644 \\
Llama 3.3 & 25.43\% & 786 \\
DeepSeek Chat & 24.72\% & 764 \\
o4-mini & 22.06\% & 682 \\
DeepSeek Reasoner & 29.44\% & 910 \\
\bottomrule
\end{tabular}
\end{table}

To assess whether these reported violations correspond to genuine reasoning defects, we further audited a subset of cases using an LLM-assisted workflow with human verification. Instead of sampling per model, we uniformly sampled 360 reported violations from all 4,613 MG violations, which satisfies Cochran's sampling criterion~\cite{cochran1968errors}. In the first stage, a strong LLM (GPT-5.4) was used as a preliminary auditor to screen whether each generated follow-up test case faithfully preserved the intended logical meaning of its transformed FOL formula. Among the 360 sampled cases, 54 were flagged as potentially problematic. In the second stage, all flagged cases were manually verified.

\Finding{Only a small fraction of reported violations are attributable to semantic drift or parsing errors, indicating that LGMT reports are highly reliable.}
\Implication{Manual inspection confirms that only 4 cases exhibit genuine semantic drift ($FP_{drift}$), and these are mainly caused by inaccuracies in the FOL-to-natural-language translation process rather than flaws in the underlying logic-grounded transformations. Across all audited cases, only one parser-induced false positive was observed. As summarized in Table~\ref{tab:rq1_reliability}, this corresponds to an empirical FRR of 1.38\%, indicating that the vast majority of reported violations reflect genuine logical reasoning defects rather than artifacts of test generation.} 

\begin{table}[pos=htbp]
\centering
\caption{Reliability of LGMT based on audited violations}
\label{tab:rq1_reliability}
\begin{tabular}{ccccc}
\toprule
Sample Size & $FP_{drift}$ & $FP_{parse}$ & FRR \\
\midrule
360  & 4 & 1 & 1.38\% \\
\bottomrule
\end{tabular}
\end{table}

Overall, these results confirm that LGMT is both effective and reliable: it reveals substantial logical reasoning defects across all evaluated models, while the low empirical FRR indicates that these reported violations overwhelmingly correspond to genuine defects rather than artifacts of test generation.


\subsection{Answer to RQ2}
\label{subsec:answer_rq2}
\setrq{2}

For RQ2, we investigate whether traditional reference-based testing overestimates the logical reasoning capabilities of LLMs by comparing it with LGMT from both apparent correctness and consistency-preserved correctness. Specifically, $Acc_{static}$ measures model performance under traditional reference-based evaluation, whereas $Acc_{cons}$ measures whether such correctness can be preserved under logic-preserving metamorphic transformations. The overall comparison, together with the Hidden Defect Rate (HDR), is summarized in Table~\ref{tab:rq2}.

\begin{table}[pos=htbp]
\centering
\caption{Comparison of Traditional Reference-Based Testing and LGMT}\label{tab:rq2}
\begin{tabular}{lccc}
\toprule
Model & Acc\_static & Acc\_cons & HDR \\
\midrule
GPT-5.2 & 73.34\% & 56.84\% & 16.50\% \\
Claude 4.5 Sonnet & 73.53\% & 60.97\% & 12.56\% \\
Llama 3.3 & 65.38\% & 52.38\% & 13.01\% \\
DeepSeek Chat & 69.91\% & 55.65\% & 14.27\% \\
o4-mini & 69.36\% & 56.20\% & 13.17\% \\
DeepSeek Reasoner & 72.44\% & 56.71\% & 15.72\% \\
\bottomrule
\end{tabular}
\end{table}

\Finding{Under traditional reference-based testing, all evaluated models achieve relatively high static accuracy (65.38\%–73.53\%), suggesting strong reasoning capability.}
\Implication{Static benchmark evaluation presents an overall optimistic view of LLM logical reasoning ability and suggests a model ranking dominated by proprietary models.}

\Finding{Under LGMT, all models exhibit a clear drop from $Acc_{static}$ to $Acc_{cons}$, with consistent accuracy decreasing to 52.38\%--60.97\%; for example, GPT-5.2 drops from 73.34\% to 56.84\%, Claude 4.5 Sonnet from 73.53\% to 60.97\%, and DeepSeek Reasoner from 72.44\% to 56.71\%.}
\Implication{Reference-based testing captures only isolated correctness on source test cases, whereas LGMT reveals that such correctness is often not logically stable under logic-preserving metamorphic transformations. This indicates that many seemingly correct predictions are achieved through coincidental correctness rather than robust logical reasoning.}

\Finding{LGMT further differentiates models in terms of hidden defects and consistency, with Claude 4.5 Sonnet achieving the highest $Acc_{cons}$ (60.97\%) and the lowest HDR (12.56\%), Llama 3.3 showing the weakest consistent performance (52.38\%), and GPT-5.2 exhibiting the highest HDR (16.50\%) despite its strong static accuracy.}
\Implication{LGMT not only lowers performance estimates, but also reshapes comparative judgments of model reliability beyond what static accuracy alone can reveal.}

Taken together, these results demonstrate that traditional reference-based testing systematically overestimates the reasoning reliability of LLMs by counting isolated correct answers as evidence of reasoning ability. In contrast, LGMT evaluates whether such correctness can survive logic-preserving transformations, revealing that a substantial portion of these correct answers are in fact logically unstable.

\subsection{Answer to RQ3}
\label{subsec:answer_rq3}
\setrq{3}

RQ3 investigates how different categories of logic-grounded MRs affect the detection rate of LGMT. 
Specifically, we analyze the Metamorphic Oracle Violation Rate (MVR) across four MR categories (MR-E, MR-S, MR-P, and MR-C) to understand which types of MRs are most challenging for LLMs. 

\begin{figure}[pos=htbp]
    \centering
    \includegraphics[width=\columnwidth]{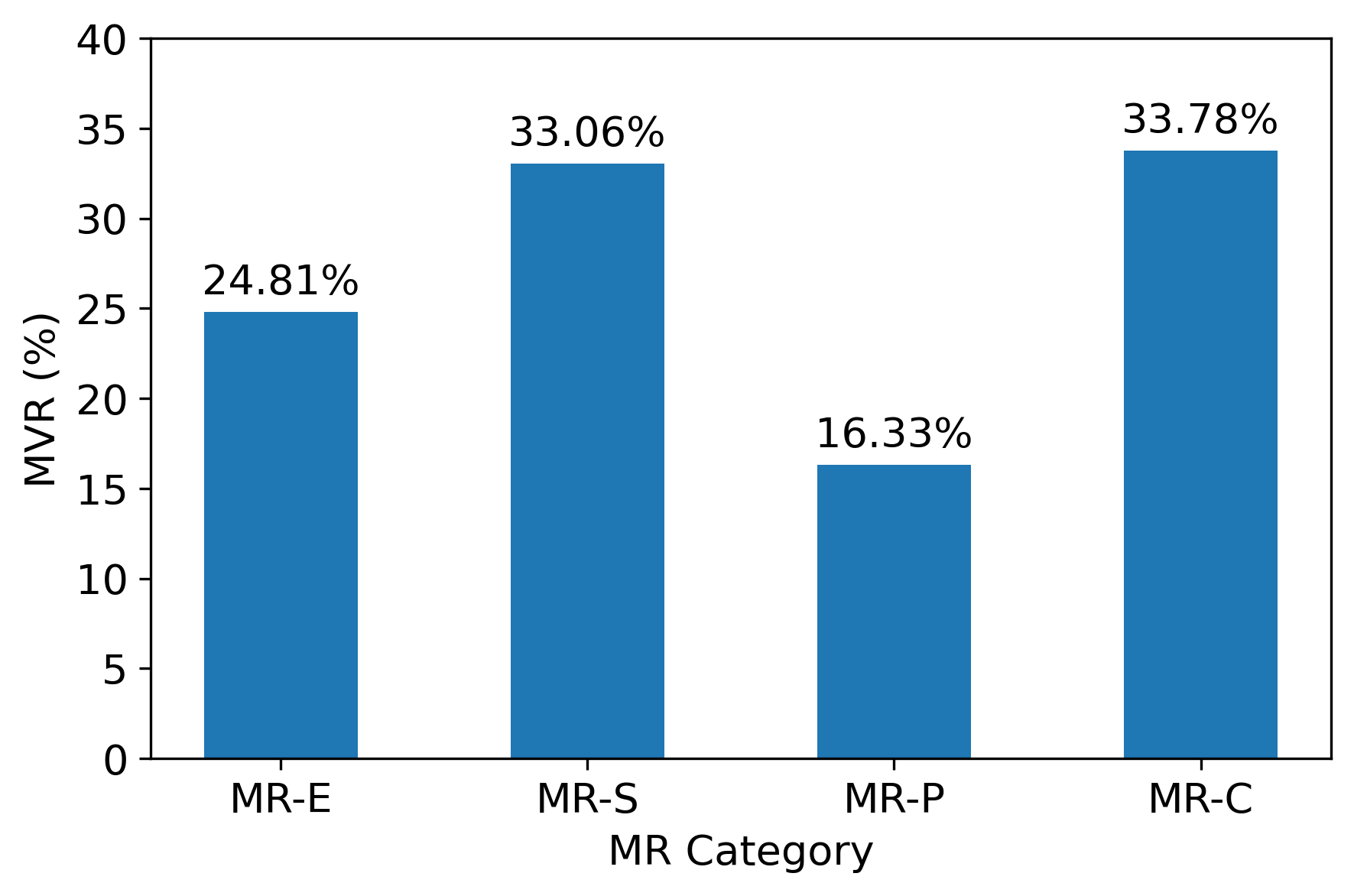}
    \caption{Average MVR across MR categories. MR-C and MR-S induce the highest inconsistency, while MR-P shows substantially lower sensitivity.}
    \label{fig:mvr_by_category}
\end{figure}

Figure~\ref{fig:mvr_by_category} presents the average MVR across MR categories, while Table~\ref{tab:rq3} provides a detailed breakdown across models.
For completeness, we further report fine-grained MVR results at the individual MR level in Appendix~\ref{app:rq3-additional}, which corroborate the category-level trends observed here.

As shown in Figure~\ref{fig:mvr_by_category}, the average MVR varies substantially across MR categories. 
MR-C and MR-S induce the highest inconsistency (33.78\% and 33.06\%, respectively), while MR-P results in the lowest MVR (16.33\%). 
MR-E exhibits a moderate level of inconsistency (24.81\%), indicating a clear separation between different types of logical transformations.

\begin{table*}[pos=tbp]
\centering
\caption{Sensitivity of LGMT across MR categories. Each cell reports the MVR. The Avg. MVR column is computed by aggregating violations across all models.}\label{tab:rq3}
\begin{tabular}{lccccccc}
\toprule
Category & Avg. MVR & GPT-5.2 & Claude 4.5 Sonnet & Llama 3.3 & DeepSeek Chat & o4-mini & DeepSeek Reasoner \\
\midrule
MR-E & 24.81\% & 25.02\% & 21.36\% & 25.02\% & 24.93\% & 24.66\% & 27.86\% \\
MR-S & 33.06\% & 32.00\% & 31.83\% & 35.25\% & 35.00\% & 27.75\% & 36.50\% \\
MR-P & 16.33\% & 19.10\% & 10.60\% & 16.10\% & 15.90\% & 14.40\% & 21.90\% \\
MR-C & 33.78\% & 39.17\% & 29.67\% & 35.17\% & 32.17\% & 26.33\% & 40.17\% \\
\bottomrule
\end{tabular}
\end{table*}

This pattern is consistently observed across all evaluated models, as shown in Table~\ref{tab:rq3}. 
For every model, MR-C and MR-S remain the most challenging categories, while MR-P consistently yields the lowest violation rates. 
This indicates that the observed sensitivity is not model-specific, but reflects a shared weakness across current LLMs.

The observed sensitivity differences can be attributed to the nature of the underlying transformations. 
The high MVR under MR-C suggests that LLMs are particularly vulnerable to equivalence-preserving transformations on conclusions, indicating instability in final decision-making. 
Similarly, the high MVR under MR-S reveals limited robustness to symbol renaming, suggesting that models may rely on lexical-level representations rather than fully abstract logical structures. 
In contrast, the comparatively lower MVR under MR-P indicates that premise-level structural transformations are less disruptive. 
However, the violation rate still exceeds 10\% across all models, highlighting that LLM reasoning remains fundamentally unstable.

\revise{A representative MR-S2 failure case makes this result intuitive. In this case, the source test concerns a text summarization model trained with machine learning algorithms. The LLM correctly predicts \textit{True} for the source case, but changes its prediction to \textit{Unknown} after MR-S2 uniformly replaces the meaningful predicates with abstract symbols, even though the logical structure and expected label remain unchanged. This failure suggests that the model does not fully understand and reason over the underlying logical structure; instead, it relies on lexical cues in meaningful predicate names when producing its output. The full source and follow-up cases are provided in Appendix~\ref{app:mrs2_failure_case}.}

\Finding{LLMs exhibit significantly different sensitivity across MR categories, with MR-C and MR-S inducing the highest inconsistency, while MR-P results in the lowest violation rate.}
\Implication{This result suggests that many current LLMs have not learned reasoning procedures that are invariant to logically equivalent reformulations. Instead, their apparent correctness may still depend strongly on how a conclusion is phrased or how symbols are instantiated.}

\Finding{The sensitivity pattern across MR categories is highly consistent among all evaluated models.}
\Implication{This suggests that the observed weaknesses are not model-specific, but reflect a fundamental limitation in current LLMs when handling logic-preserving transformations.}

Overall, these results demonstrate that LGMT not only detects logical inconsistencies, but also serves as a diagnostic tool for identifying the specific reasoning dimensions in which LLMs are most vulnerable.

\subsection{Answer to RQ4}
\label{subsec:answer_rq4}
\setrq{4}

RQ4 investigates whether advanced prompting strategies can mitigate the logical defects exposed by LGMT. 
To this end, we evaluate models under four prompting settings (Zero-shot, Zero-shot CoT, Few-shot, and Few-shot CoT) and analyze their performance using both reference-based metrics ($Acc_{static}$, $Acc_{cons}$, HDR) and LGMT-based metrics (MVR and FUR), as summarized in Table~\ref{tab:rq4}. 
Figures~\ref{fig:rq4_hdr} and~\ref{fig:rq4_fur} further visualize the distribution of HDR and FUR across models and prompting strategies.

\begin{table*}[pos=htbp]
\centering
\caption{Impact of prompting strategies on model performance. Acc\_static, Acc\_cons, and HDR are comparative evaluation metrics; MVR measures the defect detection capability of LGMT; FUR captures the limitation of consistency-based testing.}\label{tab:rq4}
\begin{tabular}{llccc|c|c}
\toprule
Model & Prompt & Acc\_static & Acc\_cons & HDR & MVR & FUR \\
\midrule
GPT-5.2 & Zero-shot & 69.07\% & 53.61\% & 15.46\% & 23.07\% & 23.33\% \\
GPT-5.2 & Zero-shot CoT & 73.34\% & 56.84\% & 16.50\% & 26.76\% & 16.40\% \\
GPT-5.2 & Few-shot & 66.13\% & 50.47\% & 15.66\% & 23.65\% & 25.88\% \\
GPT-5.2 & Few-shot CoT & 75.25\% & 63.25\% & 12.00\% & 18.57\% & 18.18\% \\
\midrule
Claude 4.5 Sonnet & Zero-shot & 73.46\% & 61.59\% & 11.88\% & 20.52\% & 17.90\% \\
Claude 4.5 Sonnet & Zero-shot CoT & 73.53\% & 60.97\% & 12.56\% & 20.84\% & 18.19\% \\
Claude 4.5 Sonnet & Few-shot & 76.87\% & 65.67\% & 11.19\% & 19.15\% & 15.17\% \\
Claude 4.5 Sonnet & Few-shot CoT & 80.03\% & 69.87\% & 10.16\% & 17.22\% & 12.91\% \\
\midrule
Llama 3.3 & Zero-shot & 67.26\% & 55.68\% & 11.58\% & 20.48\% & 23.84\% \\
Llama 3.3 & Zero-shot CoT & 65.38\% & 52.38\% & 13.01\% & 25.43\% & 22.19\% \\
Llama 3.3 & Few-shot & 69.56\% & 58.10\% & 11.45\% & 19.83\% & 22.06\% \\
Llama 3.3 & Few-shot CoT & 71.95\% & 58.59\% & 13.36\% & 23.10\% & 18.31\% \\
\midrule
DeepSeek Chat & Zero-shot & 68.36\% & 56.26\% & 12.10\% & 22.29\% & 21.45\% \\
DeepSeek Chat & Zero-shot CoT & 69.91\% & 55.65\% & 14.27\% & 24.72\% & 19.64\% \\
DeepSeek Chat & Few-shot & 71.79\% & 59.33\% & 12.46\% & 19.99\% & 20.67\% \\
DeepSeek Chat & Few-shot CoT & 72.89\% & 58.88\% & 14.01\% & 25.98\% & 15.14\% \\
\midrule
o4-mini & Zero-shot & 68.07\% & 55.29\% & 12.78\% & 22.06\% & 22.65\% \\
o4-mini & Zero-shot CoT & 69.36\% & 56.20\% & 13.17\% & 22.06\% & 21.74\% \\
o4-mini & Few-shot & 77.71\% & 63.77\% & 13.94\% & 22.39\% & 13.85\% \\
o4-mini & Few-shot CoT & 81.07\% & 67.55\% & 13.52\% & 21.81\% & 10.64\% \\
\midrule
DeepSeek Reasoner & Zero-shot & 61.76\% & 49.89\% & 11.87\% & 20.67\% & 29.44\% \\
DeepSeek Reasoner & Zero-shot CoT & 72.44\% & 56.71\% & 15.72\% & 29.44\% & 13.85\% \\
DeepSeek Reasoner & Few-shot & 77.09\% & 61.92\% & 15.17\% & 26.08\% & 12.00\% \\
DeepSeek Reasoner & Few-shot CoT & 81.82\% & 69.78\% & 12.03\% & 20.45\% & 9.77\% \\
\bottomrule
\end{tabular}
\end{table*}

\Finding{Advanced prompting strategies generally improve both static accuracy and consistent accuracy, but the improvement is not monotonic across all models.}
\Implication{Across most models, Few-shot CoT achieves the best overall performance. For example, GPT-5.2 improves from $Acc_{cons}$ 53.61\% (Zero-shot) to 63.25\% (Few-shot CoT), and DeepSeek Reasoner improves from 49.89\% to 69.78\%. However, CoT alone does not consistently yield improvements. For instance, Llama 3.3 shows a decrease from 55.68\% to 52.38\% when moving from Zero-shot to Zero-shot CoT, and DeepSeek Chat exhibits a slight drop under Few-shot CoT. This indicates that the effectiveness of prompting strategies depends on the model's capability to properly internalize and utilize reasoning instructions.}

\Finding{CoT prompting alone often increases hidden defects, while Few-shot CoT more effectively mitigates them.}
\Implication{As shown in Figure~\ref{fig:rq4_hdr}, HDR tends to increase under Zero-shot CoT across most models (e.g., GPT-5.2: 15.46\% $\rightarrow$ 16.50\%, DeepSeek Reasoner: 11.87\% $\rightarrow$ 15.72\%), suggesting that unconstrained reasoning chains may introduce additional instability. In contrast, Few-shot CoT yields the lowest HDR for several models (e.g., GPT-5.2: 12.00\%, Claude: 10.16\%), indicating that demonstration-based prompting provides a more reliable bias for structured reasoning. Nevertheless, the reduction remains limited, implying that prompting can only partially alleviate hidden defects.}

\begin{figure}[pos=htbp]
    \centering
    \includegraphics[width=\columnwidth]{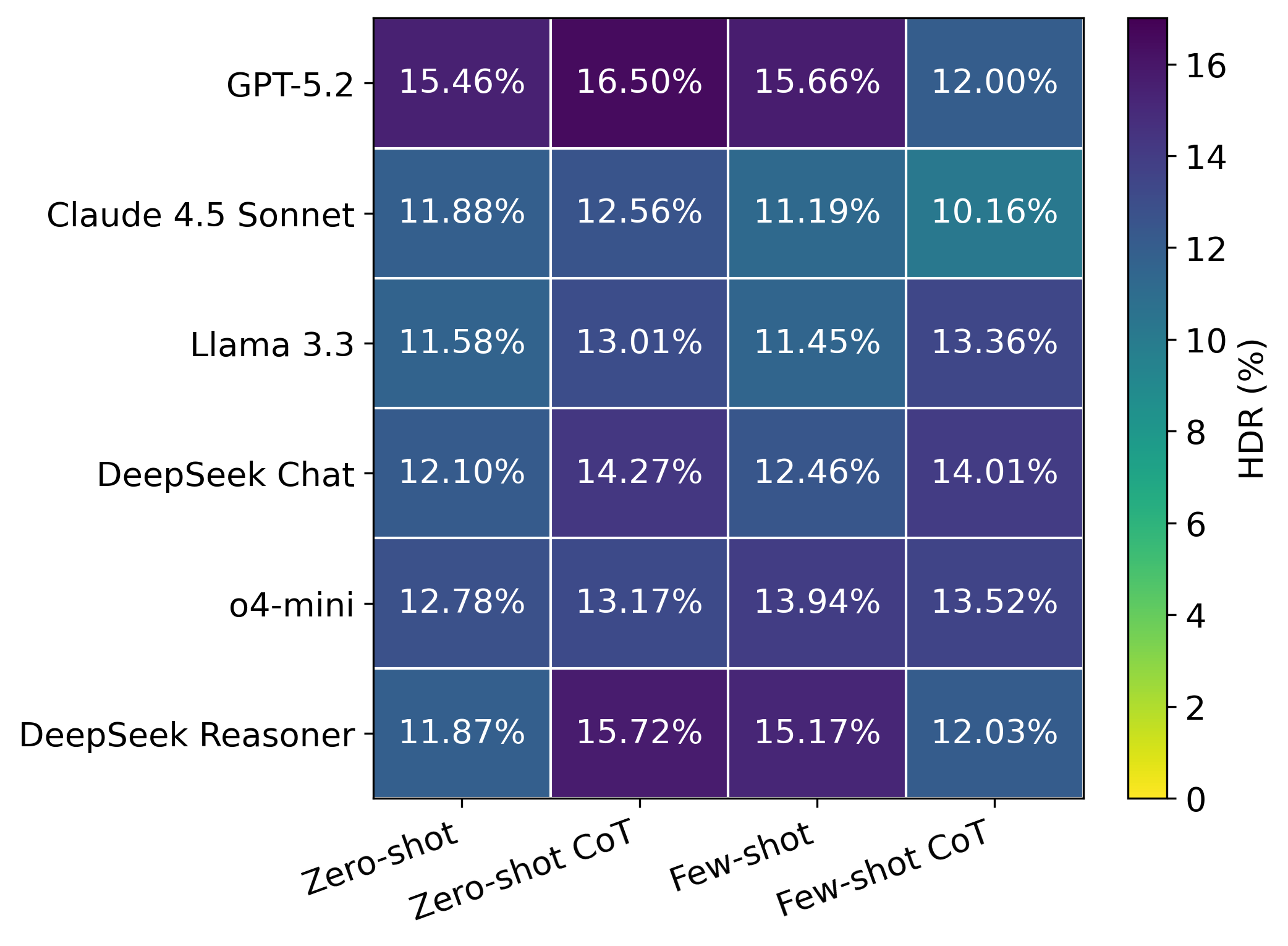}
    \caption{Hidden Defect Rate (HDR) across models and prompting strategies.}
    \label{fig:rq4_hdr}
\end{figure}

\Finding{The effectiveness of prompting in reducing metamorphic oracle violations depends on model type, with better-aligned proprietary and reasoning-optimized models benefit more from Few-shot CoT.}
\Implication{For proprietary and reasoning models (e.g., GPT-5.2, Claude, DeepSeek Reasoner), MVR increases under Zero-shot CoT but decreases under Few-shot CoT, suggesting that examples help stabilize reasoning behavior. In contrast, open-weight models such as Llama 3.3 and DeepSeek Chat exhibit increased MVR even under Few-shot CoT, indicating that additional reasoning instructions may amplify instability when the model lacks sufficient reasoning alignment. This divergence suggests that the effectiveness of advanced prompting is conditioned on the model's inherent reasoning capability.}

\begin{figure}[pos=htbp]
    \centering
    \includegraphics[width=\columnwidth]{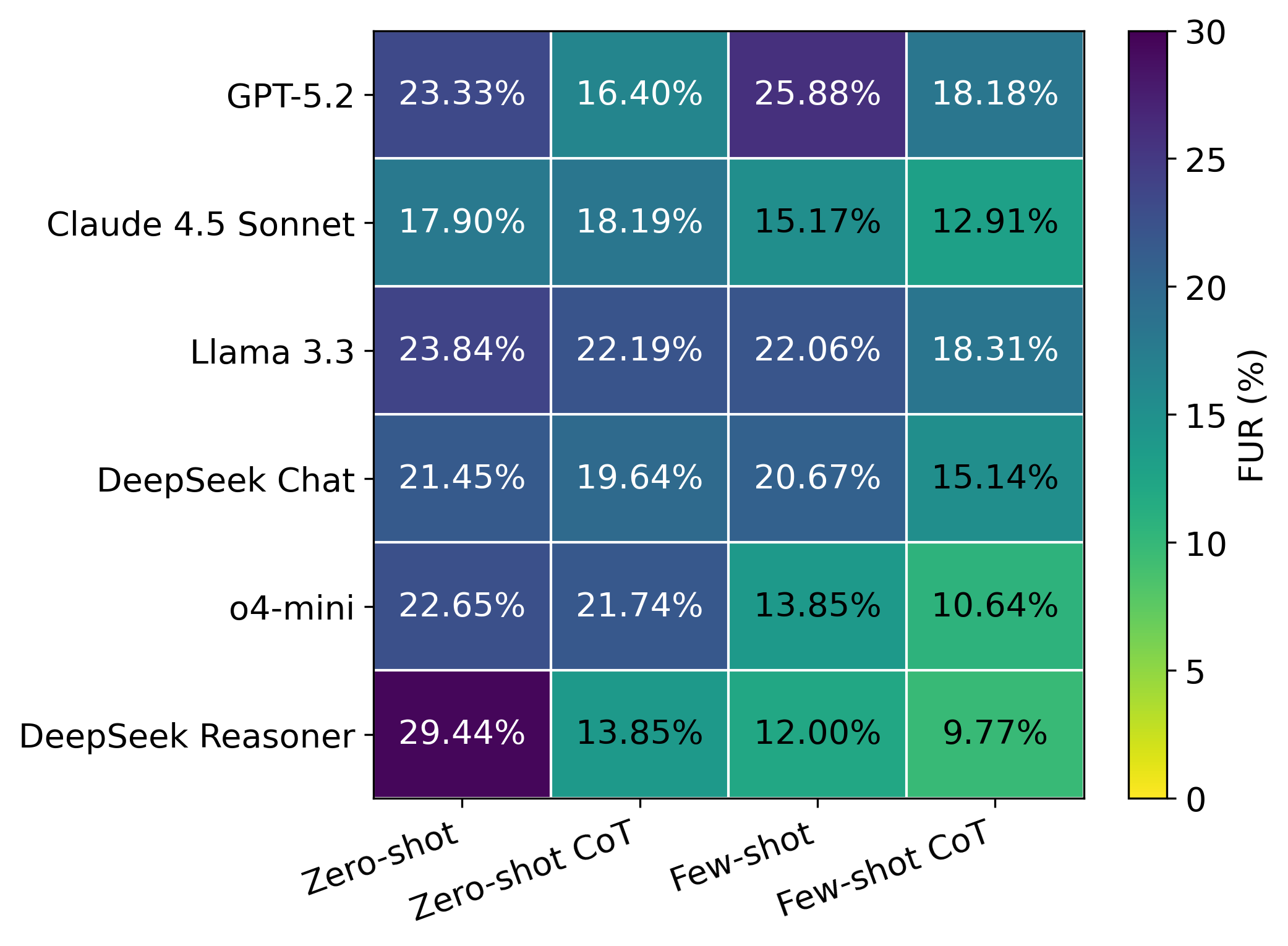}
    \caption{False Unreported Rate (FUR) across models and prompting strategies.}
    \label{fig:rq4_fur}
\end{figure}

\Finding{Few-shot CoT generally achieves the lowest FUR, but non-trivial blind spots remain.}
\Implication{As shown in Figure~\ref{fig:rq4_fur}, FUR is substantially reduced under Few-shot CoT for all models (e.g., DeepSeek Reasoner: 29.44\% $\rightarrow$ 9.77\%, o4-mini: 22.65\% $\rightarrow$ 10.64\%), indicating that fewer errors remain hidden as consistent-but-incorrect predictions. Combined with the low false report rate observed in RQ1, this suggests that LGMT is capable of exposing a large proportion of meaningful logical defects with high precision. However, FUR remains non-negligible (e.g., GPT-5.2: 18.18\%, Llama 3.3: 18.31\%), demonstrating that a significant fraction of errors still evade detection due to consistent but incorrect reasoning.}

Overall, these results indicate that while advanced prompting strategies—especially Few-shot CoT—can partially mitigate the logical defects exposed by LGMT, they do not fundamentally eliminate reasoning inconsistencies. In particular, the persistence of non-trivial HDR and FUR values suggests that improvements in prompting primarily enhance apparent and consistent performance, rather than addressing the root causes of logical reasoning failures.

\section{Discussion}
\label{sec:discussion}

\subsection{Why Do LLMs Struggle with Logical Invariance?}

Our empirical results reveal that current LLMs exhibit significant sensitivity to logic-preserving transformations, particularly under MR-S and MR-C. 
These transformations do not alter the underlying semantics, yet they consistently induce high violation rates across all models. 
This suggests that LLMs do not rely on strictly logical reasoning, but instead rely heavily on lexical-level patterns and lexical cues.

Such reliance on surface representations can explain several observations in our results. 
In particular, the high sensitivity to symbol renaming suggests that models associate meaning with specific tokens, rather than encoding the underlying relational structure.
Second, the instability under conclusion transformations suggests that final decision-making is not derived from a rigorous logical reasoning process, but is instead influenced by how the conclusion is phrased. 
Finally, the fact that Zero-shot CoT often increases inconsistency further indicates that generating longer reasoning chains does not guarantee more stable reasoning; instead, it may introduce additional opportunities for error propagation.

Taken together, these findings suggest that current LLMs approximate logical reasoning through pattern matching rather than performing robust symbolic or structure-aware inference. 
As a result, even logically equivalent formulations can lead to divergent predictions, exposing a fundamental limitation in their reasoning reliability.

\subsection{Implications of Logic-Grounded MT}

Our results highlight a key limitation of traditional reference-based testing: correctness on isolated test cases does not imply robustness under logically equivalent transformations. 
By shifting the evaluation from \textit{correctness} to \textit{invariance}, LGMT provides a fundamentally different perspective on model reliability.

In particular, LGMT evaluates whether a model can preserve its predictions under formally guaranteed logic-preserving transformations. 
This requirement goes beyond static correctness, because the model must preserve its judgment across multiple logically equivalent test cases of the same reasoning task, rather than merely answering one isolated query correctly.
As demonstrated in RQ2, many predictions that appear correct under static evaluation fail to generalize under such transformations, indicating that they arise from coincidental correctness rather than genuine reasoning.

Moreover, grounding MRs in FOL ensures that all follow-up test cases are semantically equivalent to their source test cases, thereby substantially reducing the risk of semantic drift compared with lexical MRs.
This guarantees that detected violations correspond to genuine reasoning defects, rather than artifacts of imperfect test generation.

Therefore, LGMT can be viewed as a principled and scalable alternative to traditional benchmarks, enabling systematic and oracle-free evaluation of logical reasoning robustness in LLMs.

\subsection{Limitations and Future Directions}

Despite its effectiveness, LGMT has inherent limitations. Most notably, it cannot detect cases where a model produces consistent but incorrect predictions across both source and follow-up test cases. This limitation is quantified by the FUR in RQ4, which remains non-negligible even under the strongest prompting settings. It indicates that consistency alone is insufficient to guarantee correctness, and that a portion of logical errors can still evade detection.

Future work can strengthen the testing protocol in two directions. First, \textit{multi-MR validation} can extend the current pairwise source--follow-up design by pairing a single source test case with multiple logically equivalent follow-up variants derived from different MRs. This would enable stronger groupwise consistency checking and make it more difficult for locally stable but incorrect predictions to evade detection. Second, \textit{MR composition} can apply multiple logic-grounded transformations sequentially or jointly to the same source test case. This would generate more structurally diverse yet logically equivalent follow-up cases and impose a stronger test of whether the model can preserve logical consistency under more complex reformulations.

Overall, while LGMT significantly improves the detection of logical reasoning defects, achieving truly robust logical reasoning in LLMs remains an open challenge.

\subsection{Threats to Validity}

\textbf{Construct validity.}
LGMT detects prediction inconsistency rather than absolute correctness. 
Therefore, a model that gives the same but incorrect answer to both the source and follow-up test cases may evade detection. 
We quantify this limitation using FUR. 
In addition, MVR depends on correct output parsing. 
To reduce this threat, we log parser failures and include parser-induced false positives in FRR.

\textbf{Internal validity.}
Although our symbolic transformations preserve logical semantics at the FOL level, the FOL-to-natural-language realization stage may introduce semantic drift. 
To mitigate this threat, we constrain the translation prompt to preserve logical structure and conduct an LLM-assisted audit with human verification. 
The observed FRR is low, but translation errors cannot be fully eliminated.

\textbf{External validity.}
Our evaluation focuses on FOL-style logical entailment tasks derived from FOLIO, LogicNLI, and ProverQA. 
The findings may not directly generalize to other reasoning settings, such as commonsense reasoning, mathematical problem solving, or long-horizon planning. 
Nevertheless, logical entailment provides a controlled setting for studying reasoning consistency under formally defined transformations.

\textbf{Conclusion validity.}
The results may be affected by model versions, API implementations, and prompting choices. 
We reduce randomness by using deterministic decoding whenever supported and evaluate four prompting strategies to assess prompt sensitivity. 
However, future model updates may change the absolute violation rates.

\section{Related Work}
\label{sec:related_work}

In this section, we review related work from three perspectives most relevant to our study: logical reasoning evaluation for LLMs, metamorphic testing for NLP and LLMs, and neurosymbolic approaches that combine language models with formal logic.

\subsection{Evaluating Logical Reasoning in LLMs}

Logical reasoning is a fundamental capability that directly affects the reliability of intelligent systems. 
Traditionally, the NLP community has relied on static benchmarks to evaluate this ability, including datasets such as ARC~\cite{clark2018think}, LogiQA~\cite{liu2021logiqa}, ReClor~\cite{yu2020reclor}, and FOLIO~\cite{han2024folio}. 
With the rise of LLMs, recent studies have conducted broader evaluations of reasoning ability from multiple perspectives. 
Xu et al.~\cite{xu2025are} examine LLMs from deductive, inductive, and abductive perspectives, showing that fluent generation does not necessarily imply reliable logical inference. 
Similarly, LogicBench~\cite{parmar2024logicbench} and Multi-LogiEval~\cite{patel2024multilogieval} evaluate LLMs against formal inference rules and multi-step reasoning depth. 
Other efforts, such as LogiEval~\cite{liu2025evaluating} and LogicAsker~\cite{wan2024logicasker}, further broaden the evaluation scope and reveal persistent reasoning limitations across model scales.

Another line of work constructs rule-based or proof-oriented reasoning tasks to examine whether models can perform multi-step inference over natural-language facts and rules. 
Representative datasets include RuleTaker~\cite{clark2021transformers}, ProofWriter~\cite{tafjord2021proofwriter}, EntailmentBank~\cite{dalvi2021explaining}, and CLUTRR~\cite{sinha2019clutrr}. 
These benchmarks provide controlled settings for evaluating deductive reasoning, proof generation, and relational inference. 
However, they still primarily evaluate model correctness on individual test instances with predefined ground-truth labels.

Prompting strategies have also been widely studied to improve LLM reasoning performance. 
CoT prompting~\cite{wei2022chainofthought,suzgun2023challenging}, self-consistency decoding based CoT prompting~\cite{wang2022selfconsistency}, and least-to-most prompting~\cite{zhou2023leasttomost} encourage models to generate intermediate reasoning steps or decompose complex problems into simpler subproblems. 
These methods can improve benchmark accuracy, but they do not directly ensure that model predictions remain invariant under logically equivalent reformulations. 
This motivates our RQ4, where we examine whether advanced prompting strategies can mitigate the logical defects exposed by LGMT.

Overall, existing reasoning evaluations are still largely centered on static datasets with fixed ground-truth labels. 
Such evaluations may be affected by data contamination and coincidental correctness~\cite{mondorf2024accuracy}. 
LGMT is complementary to these approaches. 
Rather than replacing benchmark-based evaluation, it introduces a consistency-based perspective by checking whether LLM predictions remain invariant under logic-preserving transformations. 
This makes it possible to reveal hidden inconsistencies that are not observable through reference-based evaluation alone.

\subsection{Metamorphic Testing for NLP and LLMs}

MT has long been recognized as an effective mechanism for alleviating the test oracle problem by checking relations among multiple executions~\cite{chen1998metamorphic,chen2021testing,chen2019metamorphic, segura2016survey,murphy2008properties,xie2011testing}. 
Rather than requiring an exact ground-truth label for every individual input, MT verifies whether the outputs of source and follow-up test cases satisfy predefined MRs. 
This makes MT suitable for testing systems, where manually specifying expected outputs for arbitrary inputs is often expensive or infeasible.

Recently, MT has gained increasing attention in NLP and LLM evaluation~\cite{cho2025metamorphic,sok2025metarag,yang2025hallucination}. 
For example, LLMORPH~\cite{cho2025metamorphic} proposes an automated MT framework for LLMs using a set of lexical MRs. 
MT has also been adapted to hallucination detection. 
MetaQA~\cite{yang2025hallucinationa} and DrHall~\cite{wu2025detecting} employ transformations such as synonym substitution or affirmative-to-negative conversion to expose factual inconsistencies in question answering. 
Similarly, Drowzee~\cite{li2024drowzee} leverages logic programming to construct test scenarios for fact-conflicting hallucination detection.

Despite these advances, many existing MT approaches for NLP still rely on lexical, syntactic, or heuristic MRs. 
Because natural language is inherently ambiguous, such MRs may introduce unintended semantic drift, which can affect the validity of follow-up test cases~\cite{cho2025metamorphic}. 
This issue is particularly important for logical reasoning tasks, where small semantic changes may alter the entailment relation.

LGMT differs from these approaches in that its MRs are derived from formal FOL equivalence laws. 
This design preserves logical semantics at the FOL level, while the natural language translation is explicitly controlled and empirically validated. 
As a result, LGMT reduces the risk of semantic drift that commonly affects heuristic MT and enables a more controlled consistency-based evaluation of LLM reasoning.

\subsection{Neurosymbolic Integration for LLMs}

Our work is also related to neurosymbolic approaches that combine LLMs with symbolic reasoning systems. 
Recent frameworks such as Logic-LM~\cite{pan2023logiclm} and LINC~\cite{olausson2023linc} adopt multi-stage pipelines in which LLMs translate natural language into formal representations and external symbolic solvers, such as Prover9 or Z3, perform the reasoning process. 
Other studies use symbolic tools to support dataset construction, reasoning supervision, or evaluation~\cite{qi2024large,cao2025advanced,jiang2025large}.

These approaches primarily use formal logic and symbolic solvers to enhance or support LLM reasoning. 
In contrast, LGMT employs formal logic as an evaluation mechanism rather than an external reasoning component. 
Instead of bypassing model limitations through symbolic solvers, LGMT uses FOL to construct logic-preserving transformations and then evaluates whether the model itself can maintain consistent judgments across logically equivalent test cases. 
In this sense, LGMT provides a black-box, consistency-based framework for diagnosing the reasoning reliability of LLMs.

\section{Conclusion and Future Work}
\label{sec:conclusion}

In this paper, we proposed LGMT, a logic-grounded metamorphic testing framework for evaluating the logical reasoning reliability of large language models. 
By deriving metamorphic relations from first-order logic, LGMT generates logic-preserving follow-up test cases and checks whether model predictions remain consistent across source and follow-up cases. 
This design reduces the semantic drift commonly introduced by lexical perturbations and enables \revise{label-free identification of reasoning defects}.

Our empirical study shows three main findings. 
First, LGMT exposes non-trivial reasoning inconsistencies across all evaluated models, while the audited false report rate remains low. 
Second, compared with traditional reference-based testing, LGMT reveals hidden defects that are not captured by static benchmark accuracy, suggesting that isolated correctness can overestimate reasoning reliability. 
Third, current LLMs are particularly sensitive to symbol-level and conclusion-level transformations. 
Although advanced prompting strategies can partially reduce these inconsistencies, they do not eliminate the observed robustness issues in our experiments.

Future work can extend LGMT in several directions. 
First, \textit{multi-MR validation} can strengthen the current pairwise design by checking consistency across multiple follow-up variants derived from different MRs. 
Second, \textit{MR composition} can generate more structurally diverse yet logically equivalent follow-up cases, enabling a stronger test of logical consistency. 
Third, integrating LGMT with lightweight symbolic verification may help reduce the remaining blind spots captured by FUR and provide a more complete assessment of logical robustness.

\section{Acknowledgments}\label{sec:acknowledgments}
The work was supported by the National Natural Science Foundation of China under Grant 62372021, the National Key R \& D Program of China under Grant 2024YFB3311503, the State Key Laboratory of Complex \& Critical Software Environment under Grant 606. We would like to thank our friend Jianxing Du for his insightful inspiration and careful verification of the proof presented in the appendix. 

\section*{Data Availability Statement}\label{sec:das}
Artifacts (code, prompts, configurations, and processed data) are available at an anonymized repository \url{https://anonymous.4open.science/r/LGMT_code-8A0D/README.md}.

\section*{Declaration of generative AI and AI-assisted technologies in the manuscript preparation process}

During the preparation of this work, the authors used ChatGPT to assist with language refinement. The authors reviewed and edited all content and take full responsibility for the final manuscript.

\printcredits


\bibliographystyle{cas-model2-names}

\bibliography{cas-refs}

\clearpage
\appendix
\section{Completeness of MR-E1 for PNNF}
\label{app:pnnf-completeness}

This appendix formalizes the notion of normalization completeness used in this work and proves it for MR-E1 under the proposed PNNF-oriented rewrite procedure.  
All discussions and results in this appendix are carried out in first-order logic over the allowed symbol set $\{\neg, \land, \lor, \rightarrow, \leftrightarrow, \forall, \exists\}$. 
Regarding the equality symbol $=$, we treat the atomic formula $t_1 = t_2$ syntactically as a binary predicate $Eq(t_1, t_2)$. 
Therefore, for the purpose of structural normalization, equality is subsumed under the definition of atomic formulas and is not listed in the allowed symbol set above.
And all formulas considered in this appendix are of finite length.

\subsection{Syntax and Notation}
We work in standard first-order logic over a fixed signature.  
Let $\mathrm{Form}$ denote the set of all first-order formulas.  
For a formula $\varphi$, let $\mathrm{FV}(\varphi)$ denote the set of free variables of $\varphi$.  
We write $\varphi[x:=y]$ for capture-avoiding substitution of $x$ by $y$.  
Two formulas $\varphi$ and $\psi$ are \emph{$\alpha$-equivalent}, written $\varphi \equiv_\alpha \psi$, if they differ only by consistent renaming of bound variables.  
A formula is in \emph{negation normal form (NNF)} if negation occurs only immediately in front of atomic formulas, and the formula contains no implication or biconditional connectives.
A formula is in \emph{prenex negation normal form (PNNF)} if it has the form $Q_1 x_1 \cdots Q_k x_k .\, M$ where each $Q_i \in \{\forall,\exists\}$ and $M$ is a quantifier-free NNF matrix.  
We define a one-step rewrite relation $\leadsto$ on formulas induced by the metamorphic relations (MRs).  
Let $\leadsto^*$ denote the reflexive transitive closure of $\leadsto$.  
Throughout this appendix, equality between formulas is understood up to $\alpha$-equivalence unless otherwise stated.

\subsection{What Completeness Means in This Work}
The completeness studied here is \emph{not} semantic completeness of first-order logic.  
Instead, it is completeness relative to a normalization target, namely PNNF.  
\begin{definition}[Completeness w.r.t.\ PNNF Normalization]
\label{def:completeness}
A set of MRs is complete with respect to PNNF normalization if:  
(i) every first-order formula can be rewritten into a PNNF formula using these MRs, and  
(ii) the obtained PNNF is unique up to $\alpha$-equivalence. 
\end{definition}
In this appendix, we focus on the proof of the completeness of MR-E1.

\subsection{The termination of MR-E1}
We define a lexicographic measure
\[
m(\varphi)=(i(\varphi),n(\varphi),q(\varphi),s(\varphi),v(\varphi),d(\varphi)) \in \mathbb{N}^6.
\]
Here $i(\varphi)$ counts the occurrences of implication and biconditional connectives in the formula;
$n(\varphi)$ counts occurrences of negation that are not applied directly to atomic formulas;
$q(\varphi)$ is the sum of depths of quantifiers that are not in the outermost prenex prefix position;
$s(\varphi)$ counts occurrences of conjunctions and disjunctions that are not flattened or not ordered according to the fixed canonical ordering of MR-E1.4;
$v(\varphi)$ counts the number of variable conflicts that block the application of MR-E1.3;
$d(\varphi)$ counts deviations in the ordering of bound variables within maximal same-quantifier blocks with respect to the canonical order enforced by MR-E1.6.
The measure $m(\varphi)$ is compared using the standard lexicographic order on $\mathbb{N}^6$.

\begin{lemma}[Strict Decrease]
\label{lem:strict-decrease}
If $\varphi \leadsto \psi$, then $m(\psi) < m(\varphi)$.
\end{lemma}

\begin{proof}
We argue by a case analysis on the applied metamorphic relation.
Each MR is oriented from non-canonical to canonical form and is designed to strictly decrease one designated component of $m$.

MR-E1.1 strictly decreases $i(\cdot)$ by eliminating implication and biconditional connectives.
MR-E1.2 strictly decreases $n(\cdot)$ by eliminating a non-atomic negation.
MR-E1.3 strictly decreases $q(\cdot)$ by lifting a quantifier toward the prenex prefix.
MR-E1.4 strictly decreases $s(\cdot)$ by flattening and reordering the formula.
MR-E1.6 strictly decreases $d(\cdot)$ by reordering bound variables within the same quantifier block.
MR-E1.5 is applicable only when a variable conflict blocks the application of MR-E1.3.
Each application of MR-E1.5 eliminates such a conflict and does not introduce new ones, and therefore strictly decreases $v(\cdot)$.

In all cases, the first component in the lexicographic measure $(i,n,q,s,v,d)$ that differs between $\varphi$ and $\psi$ strictly decreases.
Hence $m(\psi) < m(\varphi)$.
Since the lexicographic order on $\mathbb{N}^6$ is well founded and therefore cannot decrease infinitely, the rewriting process terminates.
Note that a rewrite step may increase some later components of the tuple $m(\cdot)$.
This does not affect the lexicographic comparison, because the earliest component at which $\varphi$ and $\psi$ differ strictly decreases.
\end{proof}

\begin{corollary}[Termination]
\label{cor:termination}
There exists no infinite rewrite sequence under $\leadsto$.
\end{corollary}

\subsection{Existence of PNNF Normal Forms}

\begin{lemma}[NNF Reachability]
\label{lem:nnf-reachability}
For every formula $\varphi$, there exists a formula $\varphi_{\mathrm{NNF}}$ such that
$\varphi \leadsto^* \varphi_{\mathrm{NNF}}$
and
$\varphi_{\mathrm{NNF}}$ is in NNF.
\end{lemma}

\begin{proof}
We first repeatedly apply MR-E1.1.
Each application of MR-E1.1 strictly decreases $i(\cdot)$, and therefore this process terminates with a formula $\varphi'$ such that $i(\varphi')=0$.
We then apply repeatedly MR-E1.2 to $\varphi'$.
Each application of MR-E1.2 strictly decreases $n(\cdot)$, and hence this second process terminates with a formula $\varphi_{\mathrm{NNF}}$ satisfying $n(\varphi_{\mathrm{NNF}})=0$.
Since $i(\varphi_{\mathrm{NNF}})=0$ and $n(\varphi_{\mathrm{NNF}})=0$, the resulting formula is in NNF.
\end{proof}

\begin{lemma}[PNNF Reachability from NNF]
\label{lem:pnnf-from-nnf}
For every NNF formula $\varphi$, there exists $\varphi_{\mathrm{PNNF}}$ such that
$\varphi \leadsto^* \varphi_{\mathrm{PNNF}}$
and
$\varphi_{\mathrm{PNNF}}$ is in PNNF.
\end{lemma}

\begin{proof}
Repeatedly apply MR-E1.3.
Whenever MR-E1.3 is blocked by its side condition, apply MR-E1.5 to remove the variable conflict.
Each MR-E1.3 step strictly decreases $q(\cdot)$.
Hence after finitely many steps we obtain a formula $\varphi_{\mathrm{PNNF}}$ with $q(\varphi_{\mathrm{PNNF}})=0$.
Since $\varphi$ is in NNF and MR-E1.3/MR-E1.5 do not introduce non-atomic negations, the matrix remains in NNF.
Therefore $\varphi_{\mathrm{PNNF}}$ is in PNNF.
\end{proof}

\begin{lemma}[PNNF Reachability]
\label{lem:pnnf-reachability}
For every formula $\varphi$, there exists a PNNF formula $\psi$ such that
$\varphi \leadsto^* \psi$.
\end{lemma}

\begin{proof}
Apply Lemma~\ref{lem:nnf-reachability} to obtain an NNF formula.
Then apply Lemma~\ref{lem:pnnf-from-nnf} to obtain a PNNF formula.
\end{proof}

\subsection{Uniqueness of PNNF Normal Forms}


\begin{definition}[Canonical PNNF and $\mathrm{PNNF}^*$]
A PNNF formula $\varphi$ is \emph{canonical} iff $m(\varphi) = (0,0,0,0,0,0)$.
Equivalently, $\varphi$ satisfies the canonicality conditions enforced by MR-E1.1--MR-E1.6.
We refer to canonical PNNF formulas as $\mathrm{PNNF}^*$ formulas.
We denote the set of all $\mathrm{PNNF}^*$ formulas by $\mathrm{PNNF}^*_{s}$.
\end{definition}

\begin{definition}[Normalization Procedure $NP$]
\label{def:nf}
We define a normalization procedure $NP(\cdot)$ that maps any first-order logic formula to a canonical $\mathrm{PNNF}^*$ formula by a fixed rewrite strategy:
First, apply MR-E1.1 and MR-E1.2 exhaustively to reach NNF.
Then lift quantifiers by MR-E1.3, using MR-E1.5 when alpha-renaming is needed to satisfy the side condition of quantifier lifting.
Next, apply MR-E1.4 to canonicalize the quantifier-free NNF matrix.
Finally, apply MR-E1.6 to canonically order each maximal same-quantifier block of the prefix.

\end{definition}

\begin{lemma}[Correctness of $NP$]
\label{lem:nf-correctness}
For every formula $\varphi$, we have $NP(\varphi)\in \mathrm{PNNF}^*_{s}$.
\end{lemma}

\begin{proof}
By Lemma~\ref{lem:pnnf-reachability}, the obtained PNNF is reachable from $\varphi$ by $\leadsto^*$ by applying MR-E1.1--MR-E1.3 and MR-E1.5.
MR-E1.4 canonicalizes the quantifier-free NNF part of the formula.
MR-E1.6 canonically orders bound variables within each maximal same-quantifier block.
Therefore $NP(\varphi)$ satisfies all canonicality conditions enforced by MR-E1.1--MR-E1.6.
Equivalently, we have $m(NP(\varphi))=(0,0,0,0,0,0)$.
Hence $NP(\varphi)\in \mathrm{PNNF}^*_{s}$.
\end{proof}

\begin{lemma}[Self-Consistency of $NP$]
\label{lem:np-self}
For every formula $\varphi$, the result $NP(\varphi)$ is unique up to $\alpha$-equivalence.
\end{lemma}

\begin{proof}
By construction, $NP(\varphi) \in \mathrm{PNNF}^*_{s}$.
By applying $NP(\cdot)$ to $\varphi$, the syntactic structure of $NP(\varphi)$ is fixed by the canonicality conditions induced by MR-E1.1--MR-E1.6.
Consequently, any two results obtained by applying $NP$ to $\varphi$ can differ only by $\alpha$-renaming of bound variables, and are therefore $\alpha$-equivalent.
\end{proof}

\begin{lemma}[Normalization Invariance under $\leadsto$]
\label{lem:normalization-invariance}
If $\varphi \leadsto \psi$, then $NP(\varphi) \equiv_\alpha NP(\psi)$.
\end{lemma}

\begin{proof}
A rewrite step $\varphi \leadsto \psi$ is an application of a single metamorphic relation and thus represents an alternative syntactic presentation of the same formula within the normalization system.
By Lemma~\ref{lem:nf-correctness}, the normalization procedure $NP$ terminates on any input and produces a canonical normal form that is unique up to $\alpha$-equivalence.
Therefore, applying $NP$ to $\varphi$ and to $\psi$ yields $\alpha$-equivalent results.
Hence $NP(\varphi) \equiv_\alpha NP(\psi)$.
\end{proof}

\begin{lemma}[Normalization Invariance under $\leadsto^*$]
\label{lem:normalization-invariance-star}
If $\varphi \leadsto^* \psi$, then $NP(\varphi) \equiv_\alpha NP(\psi)$.
\end{lemma}

\begin{proof}
We prove the claim by induction on the length of a rewrite sequence witnessing $\varphi \leadsto^* \psi$.
If the sequence has length $0$, then $\psi=\varphi$.
By Lemma~\ref{lem:np-self}, we have $NP(\varphi) \equiv_\alpha NP(\psi)$.
Otherwise, there exists a formula $\theta$ such that $\varphi \leadsto^* \theta$ and $\theta \leadsto \psi$.
By the induction hypothesis, $NP(\varphi) \equiv_\alpha NP(\theta)$.
By Lemma~\ref{lem:normalization-invariance}, $NP(\theta) \equiv_\alpha NP(\psi)$.
By transitivity of $\equiv_\alpha$, we obtain $NP(\varphi) \equiv_\alpha NP(\psi)$.
\end{proof}

\begin{lemma}[Fixpoint Property]
\label{lem:fixpoint-property}
If $\psi$ is a $\mathrm{PNNF}^*$ formula, then $NP(\psi) \equiv_\alpha \psi$.
\end{lemma}

\begin{proof}
Since $\psi \in \mathrm{PNNF}^*$, we have $m(\psi)=(0,0,0,0,0,0)$.
Hence no rewrite step of MR-E1.1--MR-E1.6 is applicable that would further decrease $m$.
Therefore $NP$ does not modify the syntactic structure of $\psi$.
Hence $NP(\psi) \equiv_\alpha \psi$.
\end{proof}

\begin{theorem}[PNNF-Completeness]
\label{the:pnnf-complete}
For every formula $\varphi$, there exists a unique (up to $\alpha$-equivalence) $\mathrm{PNNF}^*$ formula $\psi$ such that $\varphi \leadsto^* \psi$.  
\end{theorem}

\begin{proof}
We first prove existence.
Let $\psi := NP(\varphi)$.
By construction, $NP(\varphi)\in \mathrm{PNNF}^*$.
Since $NP$ is defined as a rewrite strategy using MR-E1.1--MR-E1.6, we have $\varphi \leadsto^* NP(\varphi)$.
Therefore $\varphi \leadsto^* \psi$ with $\psi\in \mathrm{PNNF}^*_{s}$.

We next prove uniqueness.
Assume $\varphi \leadsto^* \psi_1$ and $\varphi \leadsto^* \psi_2$ with $\psi_1,\psi_2\in \mathrm{PNNF}^*$.
By Lemma~\ref{lem:normalization-invariance-star}, we have $NP(\varphi)\equiv_\alpha NP(\psi_1)$ and $NP(\varphi)\equiv_\alpha NP(\psi_2)$.
By Lemma~\ref{lem:fixpoint-property}, we have $NP(\psi_1)\equiv_\alpha \psi_1$ and $NP(\psi_2)\equiv_\alpha \psi_2$.
Hence $NP(\varphi)\equiv_\alpha \psi_1$ and $NP(\varphi)\equiv_\alpha \psi_2$.
By transitivity of $\equiv_\alpha$, we conclude $\psi_1 \equiv_\alpha \psi_2$.
\end{proof}

\clearpage
\section{Illustrative End-to-End Example of LGMT}
\label{app:example}

We provide a concrete example to illustrate the full LGMT pipeline, including source test case construction, logic-grounded transformation, and defect detection.

\paragraph{(1) Source Test Case.}
The original reasoning instance is shown below.

\begin{quote}
\textbf{Premises}
\begin{enumerate}
    \item Lawton Park is a neighborhood in Seattle.
    \item All citizens of Lawton Park use the zip code 98199.
    \item Tom is a citizen of Lawton Park.
    \item Daniel uses the zip code 98199.
\end{enumerate}

\textbf{Conclusion}

Tom is a citizen of Washington.

\end{quote}

The ground-truth label of this instance is \textbf{Unknown}, since no premise connects Lawton Park or Seattle to Washington.

\paragraph{(2) FOL Representation.}
The corresponding symbolic representation is shown below.

\begin{quote}
\textbf{Premises}
\begin{enumerate}
    \item \texttt{NeighbourhoodIn(lawtonPark, seattle)}
    \item \texttt{forall x. (ResidentOf(x, lawtonPark) -> UseZipCode(x, num98199))}
    \item \texttt{ResidentOf(tom, lawtonPark)}
    \item \texttt{UseZipCode(daniel, num98199)}
\end{enumerate}

\textbf{Conclusion}

\texttt{ResidentOf(tom, washington)}
\end{quote}

\paragraph{(3) Applied Metamorphic Relation.}
We apply MR-E1.1 (implication elimination), which rewrites an implication into a logically equivalent disjunction:

\[
\forall x.\ (P(x) \rightarrow Q(x))
\leadsto
\forall x.\ (\neg P(x) \lor Q(x)).
\]

After transformation, the second premise becomes:

\begin{quote}
 \texttt{forall x. (not ResidentOf(x, lawtonPark) or UseZipCode(x, num98199))}
\end{quote}

\paragraph{(4) Generated Follow-up Test Case.}
The transformed FOL formula is translated back into natural language, yielding the follow-up test case shown below.

\begin{quote}
\textbf{Follow-up Premises}
\begin{enumerate}
    \item Lawton Park is a neighborhood in Seattle.
    \item For every person, either they are not a citizen of Lawton Park, or they use the zip code 98199.
    \item Tom is a citizen of Lawton Park.
    \item Daniel uses the zip code 98199.
\end{enumerate}

\textbf{Conclusion}

Tom is a citizen of Washington.
\end{quote}

\paragraph{(5) Model Outputs and Oracle Decision.}
Let the model outputs for the source and follow-up test cases be denoted as $y_s$ and $y_f$, respectively. Under LGMT, a \emph{metamorphic oracle violation} occurs if

$$y_s \neq y_f$$

Since the transformation preserves logical equivalence, the correct reasoning outcome should remain unchanged. Any inconsistency between $y_s$ and $y_f$ therefore indicates a \emph{logical reasoning defect}.

\clearpage

\clearpage
\section{Fine-Grained Sensitivity Analysis}
\label{app:rq3-additional}

As shown in Table~\ref{tab:rq3_appendix}, the fine-grained results are broadly consistent with the category-level trends reported in Section~\ref{subsec:answer_rq3}. 
In particular, several MRs under MR-S and MR-C exhibit relatively high violation rates, such as MR-S2 (34.03\%) and MR-C2 (37.50\%), further confirming that symbol-level and conclusion-level transformations are especially challenging for current LLMs. 
By contrast, most premise-level transformations remain comparatively less disruptive, with MR-P1, MR-P3, and MR-P5 all below 15\%. 
We also observe substantial heterogeneity within MR-E: while some equivalence transformations (e.g., MR-E2.3) do not have results due to the absence of applicable MG instances in our dataset, others induce much higher inconsistency (e.g., MR-E1.5 with 35.61\%), suggesting that the difficulty of formula-level transformations depends strongly on the specific rewrite type.

\begin{table*}[pos=hbt]
\centering
\caption{Sensitivity of LGMT across MRs. Each cell reports the Metamorphic Violation Rate (MVR). The Avg. MVR column is computed by aggregating violations across all models.}\label{tab:rq3_appendix}
\begin{tabular}{lccccccc}
\toprule
Category & Avg. MVR & GPT-5.2 & Claude 4.5 Sonnet & Llama 3.3 & DeepSeek Chat & o4-mini & DeepSeek Reasoner \\
\midrule
MR-E1.1 & 26.50\% & 27.00\% & 26.00\% & 27.00\% & 25.50\% & 25.00\% & 28.50\% \\
MR-E1.2 & 32.25\% & 34.00\% & 27.50\% & 32.00\% & 37.00\% & 31.00\% & 32.00\% \\
MR-E1.3 & 33.33\% & 21.43\% & 26.19\% & 33.33\% & 35.71\% & 38.10\% & 45.24\% \\
MR-E1.4 & 18.92\% & 22.50\% & 12.50\% & 17.00\% & 17.00\% & 16.50\% & 28.00\% \\
MR-E1.5 & 35.61\% & 40.91\% & 31.82\% & 40.91\% & 36.36\% & 31.82\% & 31.82\% \\
MR-E1.6 & 15.58\% & 16.50\% & 10.50\% & 12.50\% & 16.00\% & 17.00\% & 21.00\% \\
MR-E2.1 & 16.67\% & 25.00\% & 25.00\% & 16.67\% & 8.33\% & 8.33\% & 16.67\% \\
MR-E2.2 & 6.67\% & 0.00\% & 0.00\% & 13.33\% & 6.67\% & 13.33\% & 6.67\% \\
MR-E2.3 & - & - & - & - & - & - & - \\
MR-E2.4 & 29.67\% & 26.00\% & 29.50\% & 34.50\% & 28.00\% & 32.00\% & 28.00\% \\
MR-S1 & 32.08\% & 29.50\% & 28.50\% & 34.00\% & 35.50\% & 30.50\% & 34.50\% \\
MR-S2 & 34.03\% & 34.50\% & 35.18\% & 36.50\% & 34.50\% & 25.00\% & 38.50\% \\
MR-P1 & 14.83\% & 20.00\% & 8.50\% & 16.50\% & 10.00\% & 13.00\% & 21.00\% \\
MR-P2 & 15.92\% & 17.00\% & 12.50\% & 12.50\% & 19.50\% & 14.00\% & 20.00\% \\
MR-P3 & 14.83\% & 18.50\% & 8.50\% & 16.00\% & 13.00\% & 11.50\% & 21.50\% \\
MR-P4 & 22.17\% & 21.50\% & 17.50\% & 22.00\% & 24.50\% & 18.00\% & 29.50\% \\
MR-P5 & 13.92\% & 18.50\% & 6.00\% & 13.50\% & 12.50\% & 15.50\% & 17.50\% \\
MR-C1 & 30.42\% & 44.00\% & 25.50\% & 22.50\% & 25.50\% & 27.50\% & 37.50\% \\
MR-C2 & 37.50\% & 39.50\% & 34.00\% & 35.50\% & 39.50\% & 23.50\% & 53.00\% \\
MR-C3 & 33.42\% & 34.00\% & 29.50\% & 47.50\% & 31.50\% & 28.00\% & 30.00\% \\
\bottomrule
\end{tabular}
\end{table*}

\begingroup
\subsection{Representative MR-S2 Failure Case}
\label{app:mrs2_failure_case}

We provide a concrete MR-S2 failure case to illustrate why symbol-level transformations can expose reasoning instability.

\paragraph{(1) Source Test Case.}
The original reasoning instance is shown below.

\begin{quote}
\textbf{Premises}
\begin{enumerate}
    \item Machine Learning algorithms can be categorized as supervised learning, unsupervised learning, and reinforcement learning.
    \item Unsupervised learning algorithms do not require labeled data.
    \item The state-of-the-art text summarization model is trained with machine learning algorithms.
    \item Reinforcement learning is not used to train the state-of-the-art text summarization model.
    \item The Machine Learning algorithm for training text summarization models requires labeled data.
\end{enumerate}

\textbf{Conclusion}

Supervised learning is used to train the state-of-the-art text summarization model.
\end{quote}

The expected label is \textbf{True}. The tested LLM also predicts \textit{True} for this source case.

\paragraph{(2) Applied Metamorphic Relation.}
We apply MR-S2 (predicate renaming invariance), which uniformly replaces non-logical predicates with fresh abstract symbols while preserving the underlying logical structure.

\paragraph{(3) Generated Follow-up Test Case.}
After predicate renaming, the follow-up test case is shown below.

\begin{quote}
\textbf{Follow-up Premises}
\begin{enumerate}
    \item For all $x$, if $x$ has property $Pre1$, then either $x$ has property $Pre4$, or $x$ has property $Pre6$, or $x$ has property $Pre2$.
    \item For all $x$, if $x$ has property $Pre6$, then it is not the case that $x$ bears relation $Pre3$ to labeledData.
    \item For all $x$, if the state-of-the-art text summarization model bears relation $Pre5$ to $x$, then $x$ has property $Pre1$.
    \item For all $x$, if $x$ has property $Pre2$, then it is not the case that stateOfTheArtTextSummarizationModel bears relation $Pre5$ to $x$.
    \item For all $x$, if both $x$ has property $Pre1$ and stateOfTheArtTextSummarizationModel bears relation $Pre5$ to $x$, then $x$ bears relation $Pre3$ to labeledData.
\end{enumerate}

\textbf{Conclusion}

There exists at least one $x$, such that both $x$ has property $Pre4$, and stateOfTheArtTextSummarizationModel bears relation $Pre5$ to $x$.
\end{quote}

\paragraph{(4) Model Outputs and Oracle Decision.}
For the source test case, the model predicts \textit{True}. For the follow-up test case, however, it predicts \textit{Unknown}, arguing that the premises do not guarantee the required existential conclusion.
Since MR-S2 preserves the logical structure, the expected output relation is:

$$y_s = y_f$$

The observed prediction change therefore violates the metamorphic oracle and indicates sensitivity to predicate names rather than purely abstract logical structure.
\endgroup

\end{document}